\documentclass[11pt]{article}
\usepackage{times}
\usepackage{color}
\usepackage{graphicx,amssymb,amsmath}
\usepackage{amsthm}
\usepackage{xspace}
\usepackage{calc}

\usepackage[margin=1.00in]{geometry}

\newcommand{\commentout}[1]{}
\newcommand{\eat}[1]{}
\newcommand{\topic}[1]{\vspace{0.2cm}\noindent{\bf {#1}:}}

\newcommand{\calH}{{\mathcal H}}
\newcommand{\calC}{{\mathcal C}}

\newcommand{\calP}{{\mathcal P}}
\newcommand{\calQ}{{\mathcal Q}}
\newcommand{\tcalQ}{{\widetilde{\mathcal Q}}}
\newcommand{\calK}{{\mathcal K}}

\newcommand{\calE}{{\mathcal E}}

\newcommand{\calN}{{\mathcal N}}

\newcommand{\barB}{\overline{B}}
\newcommand{\R}{\mathbb{R}}

\newcommand{\Exp}{{\mathbb{E}}}
\newcommand{\Var}{\operatorname{Var}}

\newcommand{\LP}{\mathsf{LP}}

\newcommand{\pas}{\mathrm{Pas}}

\DeclareMathOperator{\tran}{Tran}
\DeclareMathOperator{\poly}{poly}
\newcommand{\lip}{1\text{-}\mathsf{Lip}}
\newcommand{\Lip}[1]{ #1 \text{-}\mathsf{Lip}}

\newcommand{\freq}{\mathsf{fq}}
\newcommand{\nfreq}{\mathsf{nfq}}

\newcommand{\tfreq}{\widetilde{\mathsf{fq}}}
\newcommand{\tnfreq}{\widetilde{\mathsf{nfq}}}

\newcommand{\mix}{\vartheta}

\newcommand{\tmix}{\widetilde{\vartheta}}

\newcommand{\hmix}{\widehat{\vartheta}}
\newcommand{\simplex}{\Delta}

\newcommand{\proj}[1]{\Pi_{#1}}

\newcommand{\B}{\mathsf{B}}

\newcommand{\tU}{\widetilde{U}}
\newcommand{\tV}{\widetilde{V}}
\newcommand{\tr}{\widetilde{r}}
\newcommand{\tg}{\widetilde{g}}

\newcommand{\tSigma}{\widetilde{\Sigma}}

\newcommand{\tA}{\widetilde{A}}

\newcommand{\tB}{\widetilde{B}}

\renewcommand{\i}{\mathrm{i}}

\renewcommand{\d}{\mathrm{d}}

\DeclareMathOperator{\Span}{Span}
\DeclareMathOperator{\sspan}{Span}
\newcommand{\supp}{\mathrm{Sp}}
\newcommand{\support}{\mathrm{Support}}

\newcommand{\innerprod}[2]{\langle #1,#2\rangle}
\newcommand{\sample}{\mathsf{s}}
\newcommand{\ep}{\widetilde{\mu}}

\newcommand{\jiannote}[1]{\textcolor{red}{#1}}

\newtheorem{reduction}{Reduction}

\newtheorem{theorem}{Theorem}[section]

\newtheorem{lemma}[theorem]{Lemma}

\newtheorem{proposition}[theorem]{Proposition}

{\theoremstyle{remark} }
{\theoremstyle{definition} }

\newenvironment{proofof}[1]{\begin{proof}[Proof of #1]}{\end{proof}}

\renewcommand{\qedsymbol}{\ensuremath{\blacksquare}}

\definecolor{olive}{rgb}{0.3, 0.4, .1}
\definecolor{fore}{RGB}{249,242,215}
\definecolor{back}{RGB}{51,51,51}
\definecolor{title}{RGB}{255,0,90}
\definecolor{dgreen}{rgb}{0.,0.6,0.}
\definecolor{gold}{rgb}{1.,0.84,0.}
\definecolor{JungleGreen}{cmyk}{0.99,0,0.52,0}
\definecolor{BlueGreen}{cmyk}{0.85,0,0.33,0}
\definecolor{RawSienna}{cmyk}{0,0.72,1,0.45}
\definecolor{Magenta}{cmyk}{0,1,0,0}

\newcommand{\eps}{\epsilon}
\newcommand{\e}{\epsilon}
\newcommand{\Dt}{\Delta}

\title{Learning Arbitrary Statistical Mixtures of Discrete Distributions}

\author{
     Jian Li\thanks{Institute for Interdisciplinary Information Sciences,
     Tsinghua University, Beijing, China 100084.
     Research supported in part by the National Basic Research Program of China grants
     2015CB358700, 2011CBA00300, 2011CBA00301, and the National Natural Science Foundation
     of China grants 61202009, 61033001, 61361136003. 
     Work performed in part at the Simons Institute for the Theory of Computing.
     Email: {\tt lapordge@gmail.com}.}
\and Yuval Rabani\thanks{The Rachel and Selim Benin School
     of Computer Science and Engineering, The Hebrew University
     of Jerusalem, Jerusalem 91904, Israel.
     Research supported by BSF grant number 2012333,
     and by the Israeli Center of Excellence on Algorithms.
     Email: {\tt yrabani@cs.huji.ac.il}.}
\and Leonard J. Schulman\thanks{California Institute of Technology,
     Pasadena, CA 91125, USA. Supported in part by NSF grant 1319745. Work performed in part
     at the Simons Institute for the Theory of Computing.
     Email: {\tt schulman@caltech.edu}.}
\and
     Chaitanya Swamy\thanks{
     Dept. of Combinatorics and Optimization, Univ. Waterloo,
     Waterloo, ON N2L 3G1, Canada. Supported in part by NSERC grant
     32760-06, an NSERC Discovery Accelerator Supplement Award, and an Ontario Early
     Researcher Award.
     Email: {\tt cswamy@uwaterloo.ca}.}
}

\date{\today}

\begin{document}
\sloppy

\maketitle
\def\thepage{}
\thispagestyle{empty}

\begin{abstract}
We study the problem of learning from unlabeled samples
very general statistical mixture models on large finite sets. Specifically, the model to
be learned, $\mix$, is a probability
distribution over probability distributions $p$, where each such $p$ is a probability
distribution over $[n] = \{1,2,\dots,n\}$. When
we sample from $\mix$, we do not observe $p$ directly, but only indirectly and in
very noisy fashion, by sampling from $[n]$
repeatedly, independently $K$ times from the distribution $p$. The problem is to infer
$\mix$ to high accuracy in
transportation (earthmover) distance.

We give the first efficient algorithms for learning this mixture model
without making any restricting assumptions on the structure of the distribution
$\mix$. We bound the quality of the solution as a function
of the size of the samples $K$ and the number of samples used.
Our model and results have applications to a variety of unsupervised
learning scenarios, including learning topic models and
collaborative filtering.
\eat{
We study the problem of learning from unlabeled samples
a statistical mixture model that is
a combination of distributions over a common large discrete domain.
More precisely, we study learning a mixture model from a collection
of independent samples
from the following source. A sample point contains
$K$ items from the set $[n] = \{1,2,\dots,n\}$. It
is generated by picking $K$ items independently from some
probability distribution $p$ on $[n]$. The distribution on items
$p$ is not fixed, but is chosen at random from a probability
measure $\mix$ on the convex hull of $k$ ``pure''
distributions on items $p^1,p^2,\dots,p^k$.
We give the first efficient algorithms for learning this mixture model
without making any restricting assumptions on the structure of
the supporting distributions $p^1,\dots,p^k$ or the distribution
$\mix$. We bound the quality of the solution as a function
of the size of the samples $K$ and the number of samples used.
Our model and results have applications to a variety of unsupervised
learning scenarios, including learning topic models and
collaborative filtering.
}
\end{abstract}

\newpage
\pagenumbering{arabic}
\normalsize

\section{Introduction}

We study the problem of learning from unlabeled samples
a statistical mixture model that is
a combination of distributions over a common large discrete domain $[n]=\{1,2,\ldots,n\}$.
This is a model that has applications to a variety of unsupervised
learning scenarios, including learning {\em topic models}~\cite{Hof99,PRTV97}
and {\em collaborative filtering}~\cite{HP99}.
For instance, in the setting of topic models,
we are given a corpus of documents, where each document is a ``bag of words''
(that is, each document is an unordered multiset of words).
The words in a document reflect the topics that this document relates
to. The assumption is that there is a small number of ``pure'' topics, where each topic is
a distribution over the underlying vocabulary of $n$ words,
and that each document is some combination of topics. 
Specifically, a $K$-word document is generated by selecting a ``mixed'' topic from a probability
distribution over convex combinations of pure topics, and then sampling $K$ words from
this mixed topic.
A good example is the so-called latent
Dirichlet allocation model of~\cite{BNJ03}, where the distribution
over topic-combinations is the Dirichlet distribution.

\vspace{-1ex}
\paragraph{The mixture model.}
In this paper, we consider {\em arbitrary} such mixtures (of a more general form), and our
goal is to learn the mixture distribution, which could be discrete, i.e., have finite
support, or continuous.
More precisely, the mixture distribution, $\mix$,
is a probability distribution over
probability distributions on $[n]$.
(Equivalently, $\mix$ is a distribution over the
$(n-1)$-simplex $\simplex_n=\{x\in \R_+^n\mid \|x\|_1=1\}$.)
When we draw a sample from $\mix$,
we obtain a distribution $p\in \simplex_n$.
However, we do not observe $p$
directly, but only indirectly and in very noisy fashion,
by sampling $K$ times independently from $p$.
Thus, our sample is a string of length $K$ over the alphabet $[n]$ where each letter is an
iid sample from $p$. We call such a sample a {\em $K$-snapshot} of $p$. (A $k$-snapshot
corresponds to a document of length $K$ in the topic-model example.) 
The problem is to learn $\mix$ with high accuracy.

Our mixture model is more general than that in the topic-model learning example,
in that we do not assume that $\mix$ is supported on the convex hull of $k$ distributions.
It is an example of a {\em statistical mixture model},
where the probability distribution from which the learning algorithm
gets samples (the mixed topic generating a document, in our topic-model example) is a
mixture of 
other probability distributions (pure topics, in our example) that are called
the mixture {\em constituents}.

\vspace{-1ex}
\paragraph{Our results.}
We give the {\em first} efficient algorithms for learning a mixture model
{\em without placing any restrictions} on the mixture.
We bound the quality of the solution as a function
of the size of the samples; clearly, larger samples give
better results.
%
A natural way to measure the accuracy of an estimate $\tmix$ in our general mixture model
is to consider the {\em transportation distance} (aka {\em earthmover distance})
between $\tmix$ and $\mix$ (see Section~\ref{sec:prelim})
where the underlying metric on distributions over $[n]$ is the {\em $L_1$ 
(or total variation) distance}. 

\eat{
More precisely, we give efficient algorithms for
learning a mixture model from a collection of independent samples
from the following source. A sample point contains,
say, $K$ items from the set $[n] = \{1,2,\dots,n\}$. It
is generated by picking $K$ items independently from some
probability distribution $p$ on $[n]$. (We call such a sample
a $K$-snapshot of $p$.) The distribution on items $p$
is not fixed, but is chosen at random from a probability
measure $\mix$ on the convex hull of $k$ ``pure''
distributions on items
$p^1,p^2,\dots,p^k$ (the constituents of the mixture).
}

Given a mixture $\mix$ supported on a $k$-dimensional subspace,
our algorithms return an estimate $\tmix$ that is
$\eps$-close to $\mix$ in transportation distance,
for any $\eps>0$, using $K$-snapshot samples for $K=K(\eps,k)$ and sample size
that is $\poly(n)$ and a suitable function of $k$ and $\eps$.
(Note that the intersection of a $k$-dimensional subspace with $\simplex_n$
could have $\exp(k)$ extreme points; so 
saying that $\mix$ lies in a $k$-dimensional subspace is substantially weaker than
assuming that $\mix$ is supported on the convex hull of $k$ points.)
Our main result (Theorem~\ref{thm:kspace}) is an efficient learning algorithm
that uses 
$O(k^4n^3\log n/\eps^6)$ 1- and 2- snapshot samples, 
and $(k/\eps)^{O(k)}$ $K$-snapshot samples, 
where $K = \widetilde{\Omega}(k^{11}/\eps^{10})=\poly(k, 1/\epsilon)$.
We also devise algorithms with different tradeoffs between the sample size and the
{\em aperture}, which is the maximum number of snapshots used per sample point (i.e.,
document size), for some special cases of the problem.
This includes, most notably, the case where $\mix$ is a
{\em $k$-spike mixture}, i.e., is supported on $k$ points in $\simplex_n$
(Theorem~\ref{thm:kspike}).
This setting has been considered previously (see below), but our algorithm
is cleaner and fits into our more general method; and more importantly, our
bounds do not depend
on distribution-dependent parameters (see the discussion below).

To put our bounds in perspective, first notice importantly that we consider transportation 
distance with respect to the $L_1$-metric on distributions.
This 
yields quite strong guarantees on the quality of our reconstruction, however working with
the $L_1$-metric (instead of $L_2$) makes the reconstruction task much harder, both
in terms of technical difficulty (see ``Our techniques'' below) and the
sample-size required:   
the $L_1$ distance between two distributions can be much larger than their
$L_2$ distance, so it is much more demanding to bound the $L_1$-error. 
In particular, this implies that the sample size must depend on $n$: 
as noted in~\cite{rabani2014learning}, with aperture independent of $n$, a sample size
of $\Omega(n)$ is {\em necessary} to recover even the expectation of the mixture
distribution with constant $L_1$-error.
The sample size needs to depend exponentially on the dimension $k$ because one can
have an $\exp(k)$-spike mixture $\mix$ (on $\simplex_n$) lying in a $k$-dimensional
subspace whose
constituents are $\Omega(1)$ $L_1$-distance apart; recovering
an $\eps$-close estimate now entails that we isolate the locations of the spikes
reasonably accurately, which necessitates $\exp(k)$ sample size.
Finally, the aperture must depend on $k$ and $\eps$. The dependence on $k$ is simply
because our learning task is at least as hard as learning $k$-spike mixtures for which
aperture $2k-1$ is necessary~\cite{rabani2014learning}. The dependence on $\eps$ is
because the lower bounds in~\cite{rabani2014learning} show that there are two (even
single-dimensional) $\ell$-spike mixtures, where $\ell=\Theta(1/\eps)$, with
transportation distance $\Omega(\eps)$ that yield identical $K$-snapshot distributions for
all $K<2\ell-1$.

A noteworthy feature of {\em all} our results is that our bounds depend
{\em only} on $n$, $k$, and $\eps$.
In contrast, all previous results for learning topic models (including those that consider
only $k$-spike mixtures)
obtain bounds that depend on distribution-dependent parameters such as some measure of the
separation between mixture constituents~\cite{PRTV97,rabani2014learning},
the minimum weight placed on a mixture
constituent, and/or the eigenvalues (or singular values) of the covariance matrix
(e.g., bounds on $\sigma_k$, or $L_1$-condition numbers, or the robustly simplicial
condition)~\cite{kleinberg2008using,AGM12,AHK12,anandkumar2012spectral}.
The distribution-free nature of our bounds is clearly a desirable feature; if the
desired accuracy is cruder than the distribution-dependent parameters, then fewer samples
are needed.

\vspace{-1ex}
\paragraph{Our techniques.}
The main result (Theorem~\ref{thm:kspace}) is derived as follows. 
First, we 
use spectral methods to compute from $1$- and $2$-snapshot samples a basis $B$ for a 
subspace $\Span(B)$ of dimension at most $k$ that nearly contains
the support of $\mix$, and such that learning the projection $\mix_B$ of $\mix$ on 
$\Span(B)$ suffices to learn $\mix$ (Section~\ref{sec:reduction}).
We need to choose $B$ carefully so as to overcome various technical challenges
that arise because we work with transportation distance in the $L_1$-metric.       
Specifically, we need to move between the $L_1$ and $L_2$ metrics at various points (the  
rotational invariance of the $L_2$-metric makes it easier to work with), and 
to avoid a $\sqrt{n}$-factor distortion due to this movement, we need to establish that an 
$L_1$-ball in $\Span(B)$ is close to being an $L_2$-ball 
in $\Span(B)$ (see Lemma~\ref{lm:vectorinB}).   
This allows one to argue that: (a) $\mix_B$ is supported in an $L_2$-ball of radius
$O\bigl(\frac{1}{\sqrt{n}}\bigr)$, which makes it feasible to learn it within
$L_2$-error $\frac{\e}{\sqrt{n}}$ (and hence $L_1$-error $\e$); and (b) projecting this  
reconstructed mixture to $\Dt_n$ preserves the $L_1$-error (up to a 
$\poly\bigl(k,\frac{1}{\e}\bigr)$ factor).  
We remark that the standard SVD technique does not suffice for our purpose, since the
resulting subspace need not satisfy the above ``spherical'' property of $L_1$-balls (see 
also the discussion in Section~\ref{sec:reduction}).
%
Next, we define a projection of the $K$-snapshot samples using $B$.
We compute the estimate $\tmix_B$ of $\mix_B$ by
averaging the projections and transforming the result to $\Span(B)$ (see
Section~\ref{sec:directlearn}).
The proof relies on large deviation bounds. One can
show that $\mix$ is close to $\mix_B$. 
The output $\tmix_B$ converges to this projection as the number
of samples grows. The rate of convergence can be bounded using
tools from approximation theory.

\smallskip
The result for the special case of $k$-spike mixtures (i.e., $\mix$ is supported on $k$
distributions) 
uses a three-step approach analogous to the argument
in~\cite{rabani2014learning}, but the implementation of each
step is different). The first step finds $B$ as in the general case.
In the second step (Section~\ref{subsec:project1d}), the algorithm projects the sample data onto
the basis vectors in $B$. From this data, the algorithm computes
a good approximation to the projection of $\mix$ onto each
axis. The idea is to use linear programming to compute a piecewise
constant discretization of the projected measure such that the first
$K$ moments are close to the empirical
moments derived from the samples of $K$-snapshots.
The analysis uses a classical result in approximation theory due
to Jackson that estimates the error in approximating
a $1$-Lipschitz function on $[0,1]$ by the first $K$ Chebyshev
polynomials. (In fact, this step, too, does not use the special
structure of the mixture. It works in the case of an arbitrary
measure $\mix$, and our error estimates are asymptotically optimal
in general.) In the third step (Section~\ref{subsec:reconstruct}), we use the approximate
projected measures to compute a good approximation for the
projection of $\mix$ on $\Span(B)$, giving our algorithm's
output. The main idea here is similar to that of the second step.
We discretize the projection and use a linear program to compute
a discretized measure whose projections onto the axes used in
the second step give a good match to the computed approximations
on those axes. The analysis of this algorithm uses Yudin's
multidimensional generalization of
Jackson's theorem~\cite{yudin1976multidimensional}. Both the
second step and the third step use Kantorovich-Rubinstein duality
to relate the results from approximation theory to the approximation
guarantees in terms of the transportation distance.

\vspace{-1ex}
\paragraph{Related work.}
Generally speaking, our problem is an example of learning a
{\em mixture model}. Unlike our case, other mixture learning
problems, such as learning a mixture of
Gaussians (see~\cite{Das99,BS10,MV10}), assume
a special structure of the distributions that contribute to the
mixture. We discuss this related literature below.

A few previous papers consider the problem of learning a topic
model~\cite{AGM12,AFHKL12,AHK12,rabani2014learning}.
They all make
limiting assumptions on the structure of the mixture model. The
only paper that considers an arbitrary distribution $\mix$
over combinations
of topics is~\cite{AGM12}. However, this paper assumes that the
pure topics are $\rho$-separated, which means that each topic
has an anchor word that has probability at least $\rho$ in this
topic, and probability $0$ in any other topic. In the case of an
arbitrary $\mix$ (over such topics),
the paper~\cite{AGM12} learns the correlation matrix for pairs of
pure topics and not $\mix$.
In the special case of latent Dirichlet allocation, the paper also
reconstructs $\mix$. The latent Dirichlet allocation
setting is also considered in~\cite{AFHKL12}. For this special case,
they relax the condition in~\cite{AGM12} to the requirement that
the matrix whose columns are the word distributions of the $k$
pure topics has full rank $k$. The constraints on the model that
are imposed in~\cite{AGM12,AFHKL12} allow them to achieve
their learning goals using documents of constant size that is independent
of the number of pure topics $k$ and the desired accuracy $\eps$.
As we show in this paper, this is impossible in the general case.
The remaining two papers mentioned above~\cite{AHK12,rabani2014learning}
consider only the case where each document is generated from
a single pure topic, so $\mix$ is a discrete distribution with
support of size $k$. The first paper~\cite{AHK12} imposes on the
pure topics the same rank condition as in~\cite{AFHKL12}, and
thus is able to learn the model from constant size documents.
The second paper~\cite{rabani2014learning} studies the general pure topic
documents case and shows how to learn the model from documents
of size $2k-1$, which is a tight requirement. Notice that in this case,
the document size is independent of the desired accuracy. Our
results specialized to this case are motivated by the techniques
in~\cite{rabani2014learning}. They give a simpler and cleaner proof
that roughly matches the results there (in particular, the mixture
model is recovered using $K$-snapshots for $K=2k-1$, which is
optimal).

Learning statistical mixture models has been studied in the theory
community for about twenty years. The defining problem
of this area was the problem of learning a mixture of high-dimensional
Gaussians. Starting with the ground-breaking result of~\cite{Das99},
a sequence of improved
results~\cite{DS00,AK01,VW02,KSV05,AM05,FOS06,BV08,KMV10,BS10,MV10}
resolved the problem. Beyond Gaussians, various recent papers
analyze learning other highly structured mixture models (e.g., mixtures of
discrete product
distributions)~\cite{KMRRSS94,FreundM99,CryanGG02,BGK04,MR05,DHKS05,
FeldmanOS05,KSV05,CHRZ07,CR08a,CR08b,DaskalakisDS12}.
An important difference between this work and ours is that the structure of the
mixtures that they discuss enables learning using samples that consist of a
$1$-snapshot of a random mixture constituent (which is impossible in our setting).
Since Gaussians and other structured mixtures can be learned from
1-snapshot samples, the issue of the samples themselves being
generated from a combination of the mixture constituents does not arise
there. Our problem is unique to learning from multi-snapshot samples.

\section{Preliminaries and notation} \label{sec:prelim}
Let $T: X\rightarrow Y$
be a transformation from a normed space $X$ (with norm $\|\cdot\|_X$) to a normed space $Y$ (with norm $\|\cdot\|_Y$).
Let $\mu$ be a measure defined over $X$.
We use $\mu\circ T^{-1}$ to denote the image measure (or pushforward measure) defined over $Y$:
$\mu\circ T^{-1}(U)=\mu(T^{-1}(U))$ for all measurable $U\subset Y$.
It is a simple fact that (see e.g., \cite{dudley2002real}) that for any measurable function $f$,
\begin{align}
\label{eq:change}
\int_Y f \,\d(\mu \circ T^{-1})=\int_X f\circ T \,\d \mu.
\end{align}
For ease of notation, we sometimes write $T\mu$ to denote the image measure $\mu\circ T^{-1}$.
For a vector $v$, we use $\|v\|$ to denote its $L_2$ norm, and for an operator $T$,
we use $\|T\|_{X\rightarrow Y}$ to denote its
operator norm
(i.e., $\|T\|_{X\rightarrow Y}=\sup\{\|Tx\|_Y \mid x\in X, \|x\|_X=1\}$).
For ease of notation, we use $\|T\|$ to denote the $L_2\rightarrow L_2$ operator norm of $T$.

\topic{Transportation Distance}
Let $(X, d)$ be a separable metric space.
Recall that for any two distributions $P$ and $Q$ on $S$, the transportation distance $\tran(P,Q)$ (also called Rubinstein distance, Wasserstein distance
or earth mover distance in literature)
is defined as
\begin{align}
\label{eq:transprimal}
\tran(P,Q):= \inf \left\{\int d(x,y) \,\d\mu(x,y) : \mu\in M(P,Q) \right\},
\end{align}
where $M(P,Q)$ is the set of all joint distributions (also called coupling) on $X\times X$ with marginals $P$ and $Q$.
For the discrete case (say $X$ is a finite set of discrete points $v_1,\ldots, v_n$),
\eqref{eq:transprimal} is in fact the following familiar transportation LP:
\eat{
\begin{align*}
\text{ minimize } & \quad\sum_{i,j} d(v_i,v_j) x_{ij} \\
\text{ subject to } & \quad\sum_{j} x_{ij} = P(\{v_i\}),\,\,\, \forall i\in [n], \quad \sum_{i} x_{ij} = Q(\{v_j\}), \,\,\,\forall i\in[n], \quad x_{ij}\in [0,1] \,\,\forall i\in[n],j\in[n].
\end{align*}
}
$
\text{ minimize }  \sum_{i,j} d(v_i,v_j) x_{ij} \,\,
\text{ subject to }  \sum_{j} x_{ij} = P(\{v_i\}),\,\,\, \forall i\in [n], \sum_{i} x_{ij} = Q(\{v_j\}), \,\,\,\forall i\in[n],  x_{ij}\in [0,1] \,\,\forall i\in[n],j\in[n].
$
Any feasible solution $\{x_{ij}\}_{i,j}$ of the above LP is in fact a coupling of $P$ and $Q$, since
it can be interpreted as a joint distribution over $X\times X$, and the constraints of the LP dictate
the first marginal of $\{x_{ij}\}$ is $P$ and the second is $Q$.

Suppose $\mu$ is a measure on some metric space $(X, d)$.
Let $T: X\rightarrow X$ be an operator.
$T$ naturally defines a coupling $W$ between $\mu$ and the image measure $T\mu$:
for any $R\subseteq X\times X$, let $W(R)=\mu(\{x\mid (x, Tx)\in R\})$
(so for any measurable $S\subseteq X$, $W(S\times T(S))=\mu(S)$).
For ease of description, for such a coupling, we often say ``we couple
$x$ with $Tx$ together".

Let $\lip$ be the set of 1-Lipschitz functions on $X$, i.e.,
$\lip:=\{f: X\rightarrow \R \mid |f(x)-f(y)|\leq d(x,y) \text{ for any }x,y\in X\}$.
We need the following important theorem by Kantorovich and Rubinstein (see e.g., \cite{dudley2002real}):
\begin{align}
\label{eq:trans}
\tran(P,Q)=\sup\left\{\left|\int f \d(P-Q)\right |: f\in \lip\right\}.
\end{align}
In the discrete case, Kantorovich-Rubinstein theorem is exactly LP-duality
(the dual of the aforementioned LP is:
$
\text{ maximize }  \sum_{i} f_i (P(\{v_i\})-Q(\{v_i\})), \,
\text{ subject to }  f_i-f_j \leq d(v_i,v_j) \,\forall i\in [n], j\in[n].
$
).

It is important to notice the transportation distance and the Lipschitz condition are
associated with the same metric $d(x,y)$.
We use $\tran_1$ and $\tran_2$
to denote the transportation distance for $L_1$ and $L_2$ metrics respectively.
In 1-dimensional space, $L_1$ and $L_2$ are the same and we simply use $\tran$.
The following simple lemma will be useful in several places.
The proofs are standard; we include them in Appendix~\ref{app:prelim} for
completeness.

\begin{lemma}
\label{lm:proj3}
$(X, \|\cdot\|_X)$ and $(Y,\|\cdot\|_Y)$ are two normed spaces.
We are given two probability measures $\mu, \nu$ defined over $X$ such that
$\tran(\mu,\nu)\leq \epsilon$.
\begin{list}{{\upshape (\roman{enumi})}}{\usecounter{enumi} \topsep=0ex \itemsep=0ex
    \addtolength{\leftmargin}{-1ex} \addtolength{\labelwidth}{\widthof{(iii)}}}
\item
Suppose $T: X\rightarrow Y$ is a transformation from $X$ to $Y$.
$
\tran(T\mu,T\nu) \leq \epsilon\cdot \|T \|_{X\rightarrow Y}.
$
\item Furthermore, if both $\mu$ and $\nu$ are supported on a subspace $V\subset X$, then
$
\tran(T\mu,T\nu) \leq  \epsilon\cdot \|T \|_V,
$
where $\|T\|_V = \sup_{x\in V}\|T x \|_Y/ \|x\|_X$.
\item We are given two operators $T$ and $T'$ such that $\|T-T'\|_{X\rightarrow Y}\leq \epsilon$.
Suppose $\|T\|_{X\rightarrow Y}=O(1)$ and $\|x'\|_X=O(1)$ for all $x'\in \support(\nu)$.
Then, we have that
$\tran(T\mu, T'\nu)\leq O(\epsilon)$.
\end{list}
\end{lemma}

We state the following
standard Chernoff-Hoeffding bound and Bernstein inequality.

\begin{proposition}\label{prop:chernoff}
Let $X_i (1 \leq i \leq n)$ be independent random variables with values in $[0,1]$.
Let $X =  \sum_{i=1}^n X_i$. 
For every $t > 0$, we have that
$\Pr\Big[\,|X-\Exp[X]| >  t\Big] < 2\exp(-2t^2/ n)$.
\end{proposition}

\begin{proposition}
\label{prop:bernstein}
Let $X_i (1\leq i\leq n)$ be independent random variables with $\|X_i\|\leq 1$, $\Exp[X_i]=0$
for all $i$.
Let $X =  \sum_{i=1}^n X_i$.
Let $\sigma^2=\Var[X]=\sum_{i=1}^n \Var[X_i]$.
Then,
$\Pr\big[\,|X|>t\,\big] \leq 2\exp\left(-\frac{t^2}{2(\sigma^2+t/3)}\right)$.
\end{proposition}

We will use the following results from the matrix perturbation and random matrix theory.
\begin{theorem}
\label{thm:wedin}
(Wedin's theorem, see e.g., \cite[pp.261]{stewartmatrix})
Let $A, \tA \in \R^{m\times n}$ with $m\geq n$ be given.
Let the singular value decompositions of $A$ and $\tA$ be
$$
(U_1,U_2,U_3)^T A (V_1,V_2) =
\left(
\begin{array}{cc}
\Sigma_1  & 0  \\
0  &   \Sigma_2 \\
 0 &   0
\end{array}
\right),
\quad
(\tU_1,\tU_2,\tU_3)^T \tA (\tV_1,\tV_2) =
\left(
\begin{array}{cc}
\tSigma_1  & 0  \\
0  &   \tSigma_2 \\
 0 &   0
\end{array}
\right)
$$
Let $\Phi$ be the matrix of canonical angles between $\sspan(U_1)$ and $\sspan(\tU_1)$
and $\Theta$ be that between $\sspan(V_1)$ and $\sspan(\tV_1)$.
If there exists $\delta, \alpha>0$ such that $\min_i \sigma_i(\tSigma_1)\geq \alpha +\delta$
and $\max_i \sigma_i(\Sigma_2)\leq \alpha$, then
$\max\{\| \sin \Phi\|, \| \sin \Theta\|\} \leq \frac{\|A-\tA \|}{\delta}$.
Moreover, 
$\| \proj{A} - \proj{\tA} \| = \|\sin \Phi \|$
(see e.g., \cite[pp.43]{stewartmatrix}).
\end{theorem}

\begin{theorem}[\cite{vu2005spectral}]
\label{thm:vu2005spectral}
For every constant $c>0$, there is a constant $C>0$ such that
the following holds.
Let $A$ be a symmetric with entries $a_{ij}=a_{ji}=X_{ij}$, where
$X_{ij}$, $1\leq i\leq j\leq n$ are independent random variables.
Suppose each $X_{ij}$ is such that $|X_{ij}|<K$, $\Exp[X_{ij}]=0$ and
$\Var[X_{ij}]\leq \sigma^2$ where $\sigma\geq C^2K\ln^2 n/\sqrt{n}$.
Then, it holds that
$$
\Pr[\|A\|\leq 2\sigma \sqrt{n} +C(K\sigma)^{1/2} n^{1/4}\ln n ] \geq 1- 1/n^c.
$$
\end{theorem}

The Chebyshev polynomial (of the first kind) is defined as the polynomial satisfying
$T_n(\cos(x)) = \cos(nx)$. 
An equivalent recursive definition is:
$T_0(x)=1, T_1(x)=x$ and $T_{n+1}(x)=2xT_n(x)-T_{n-1}(x)$.
We need the classical Jackson's theorem (see e.g., \cite{rivlin2003introduction}) in approximation theory
(specialized to our setting) and a 
multidimensional generalization of Jackson's theorem established by Yudin
\cite{yudin1976multidimensional} (Theorem~\ref{thm:yudin}).

\begin{theorem}[Jackson's Theorem]
\label{thm:jackson}
It is possible to approximate any function on $[0,1]$ in $\lip$
within $L_\infty$ error $O(1/K)$ using Chebyshev polynomials (or equivalently
trigonometric polynomials) of degree at most $K$, i.e.,
there exist $\{t_i\}_{i\in [K]}$ such that
$f(x)=\sum_{i=0}^K t_i T_i(x) \pm O(1/K)  \quad \forall x\in [0,1]$.
Moreover, $|t_i|\leq\poly(K)$ for all $i=0,\ldots,K$.
\end{theorem}

\begin{theorem}
\label{thm:yudin}
We use $\B_2^h(R)$ to denote the sphere $\{x\in \R^h\mid \|x\|_2\leq R\}$.
For any function
$f: \B_2^h(1)\rightarrow \mathbb{C}$ which is $\lip$ (in $L_2$ distance),
there exists complex numbers $c(t')$ for $t'\in \mathbb{Z}^{h} \cap \B_2^h(R)$,
such that $|c(t')|\leq \exp(O(h))$%
\footnote{
In Yudin's theorem,
$c(t')$ is in fact $\hat{f}(t') \lambda(t'/R)$, where
$
\hat{f}(t') = {1 \over (2 \pi)^h} \int_{x\in [-\pi,\pi]^{d}} f(x) e^{-\i \innerprod{t'}{x}}\, \d x
$
is the Fourier coefficient,
$\lambda(x)=(\phi \ast \phi)(x)$, 
$\phi(x)$ is
the first normalized eigenfunction of a PDE known as Helmholtz equation, and
$\ast$ is the convolution.
}
and for all $x\in \B_2^h(1)$,
$$
\Bigl| f(x)-\sum_{t'\in \mathbb{Z}^{h} \cap \B_2^h(R)} c(t') e^{\i\innerprod{t'}{x}}\Bigr|\le O\left( {h \over R}\right).
$$
\end{theorem}

\eat{
We use $\B_2^h(R)$ to denote the sphere $\{x\in \R^h\mid \|x\|_2\leq R\}$.
Let $\mu$ and $\phi(x)$ be
the first eigenvalue and normalized eigenfunction of the following partial differential equation defined in $\B_2^h(1)$ (also known as Helmholtz equation):
$$
\Delta u = \sum_{i=1}^h \partial^2 u /\partial x_i^2  = - \mu u, \quad \mu\geq 0, \quad u|_{\partial \B_2^h(1/2)}=0.
$$
It is well known result in partial different equation that the first eigenvalue
$\mu=\Omega(h)$ is the first positive zero of the Bessel function $J_{m/2-1}$.
Let $\lambda(x)=(\phi \ast \phi)(x)$ where $\ast$ is the convolution.
For a function
$f: \B_2^h(1)\rightarrow \mathbb{C}$,
Let
$$
U_R f(x) = \sum_{t'\in \mathbb{Z}^{h} \cap \B_2^h(R)} \hat{f}(t') \lambda(t'/R) e^{\i \innerprod{t'}{x}},
$$
where
$
\hat{f}(t') = {1 \over (2 \pi)^h} \int_{x\in [-\pi,\pi]^{d}} f(x) e^{-\i \innerprod{t'}{x}}\, \d x
$
is the Fourier coefficient.
It is known that $U_R f(x)$ converges uniformly to $f(x)$ with the following rate~\ref{yudin1976multidimensional}:
If $f\in \lip$ (in $L_2$ distance), then
\begin{align}
\label{eq:fourier}
|f(x)-(U_R f)(x)|\le O\Bigl( { \mu \over R}\Bigr)=O\Bigl( {h \over R}\Bigr) \ \ \ \text{for }x\in \B_2^h(1).
\end{align}
}

\section{Learning single-dimensional mixtures: the coin problem}
\label{sec:coin}
In this section, we consider the problem of 
learning a mixture $\mix$ supported on $[0,1]$, which we call the {\em coin problem}.
Using results in~\cite{rabani2014learning}, these results carry over to the setting
where $\mix$ supported on a line segment in the $(n-1)$-simplex
$\simplex_{n}=\{x\in\mathbb{R}^{n}_{\geq 0}, \|x\|_1= 1\}$. We first consider an arbitrary
(even continuous)
$\mix$ in $[0,1]$; in Section~\ref{subsec:kspikeoned}, we consider the case where $\mix$ is a
{\em $k$-spike mixture}. 

Let
$
B_{i,K}(x)={K\choose i} x^i (1-x)^{K-i}.
$
Let $N_K$ denote the number of $K$-snapshots we take from $\mix$.
For $0\leq i\leq K$, define
$\freq_i(\mix):=\int B_{i,K}(x) \d \mix$.
We call $\freq(\mix):=\{\freq_i(\mix)\}_{0\leq i\leq K}$
the {\em frequency vector} corresponding to $\mix$.
We use $\tfreq_i$
to denote the fraction of sampled coins that showed ``heads'' exactly $i$ times
and let $\tfreq:=\{\tfreq_i\}_{0\leq i\leq K}$ be the {\em empirical frequency vector}.
It is easy to see that $\freq(\mix)=\Exp[\tfreq]$.
If we take enough samples, the
frequency vector corresponding to the empirical measure $\tmix$ should be
sufficiently close to that of $\mix$.

\begin{lemma}
\label{lm:closefreq}
By taking  $N_K= \kappa^{-2}\log(K/\delta)$ samples,
with high probability $1-\delta$, we have that
$\|\, \freq(\mix)-\tfreq \,\|_\infty \leq \kappa$.
\end{lemma}

\begin{proof}
Using Chernoff bound (Proposition~\ref{prop:chernoff}),
we can see that
$
\Pr[ |\freq_i(\mix)-\tfreq_i| >\kappa] \leq 2 \exp(-2\kappa^2 N_K)\leq \delta/K.
$
Then the lemma follows from a simple application of union bound over all $K+1$ coordinates.
\end{proof}


\begin{theorem}
\label{thm:naive}
There exists an algorithm, with running time polynomial in $K$, that gets
as input $m=\poly(K)$ coins from a mixture $\mix$, each tossed $K$ times,
and output a mixture $\hmix$ such that $\tran(\mix,\hmix)\leq O(1/\sqrt{K})$
with high probability.
\end{theorem}

Theorem~\ref{thm:naive} can be proved by a simple application of Chernoff bound (where we
set $\hmix(\{\frac{i}{K}\})=\tfreq_i$), which we omit here.
We provide an alternative proof based on Bernstein polynomials later.
It is a natural question to ask whether
$O(1/\sqrt{K})$ in Theorem~\ref{thm:naive} achieves the optimal
aperture-transportation distance tradeoff.
In \cite{rabani2014learning},
it is shown that
recovering a $K$-spike mixture within transportation distance $O(1/K)$
using $c(2K-1)$ (for any constant $c\geq 1)$
aperture requires $\exp(\Omega(K))$ samples.
The following theorem provides a {\em matching upper bound}.

\begin{theorem}
\label{thm:main}
There exists an algorithm, with running time polynomial in $K$, that gets
as input $m=\exp(O(K))$ coins from a mixture $\mix$, each tossed $K$ times,
and outputs a mixture $\hmix$ such that $\tran(\mix,\hmix)\leq O(1/K)$
with high probability.
\end{theorem}

To prove Theorem~\ref{thm:main}, we make a crucial observation (Lemma~\ref{lm:freqtotran1}) that links
the transportation distance, the frequency vector and the coefficients of Bernstein polynomial approximation.
Lemma~\ref{lm:approxpoly} bounds these coefficients using
the relation between Bernstein polynomial basis and Chebyshev polynomial basis.
We then provide a simple LP-based algorithm to reconstruct $\mix$.

\begin{lemma}
\label{lm:freqtotran1}
Suppose for any $f\in \lip[0,1]$, there exist $K+1$ real numbers
$c_0,\ldots,c_K \in [-C,C]$, for some value $C>0$
and $\lambda>0$,
such that
$f= \sum_i c_i B_{i,K} \pm O(\lambda)$.
Then for any two distributions $P$ and $Q$ on $[0,1]$,
$
\tran(P,Q)\leq C\cdot \|\, \freq(P)-\freq(Q) \,\|_1  + O(\lambda).
$
\end{lemma}
\begin{proof}
We have $\freq_i(P)=\int B_{i,K} \,\d P$.
For any $f\in \lip$ such that $f(x)\in [0,1]$ for all $x\in [0,1]$, we have
\begin{align*}
\left| \int f \d(P-Q) \right| &=\left| \sum_{i=0}^K c_i \int B_{i,K} \,\d(P-Q) \right|+ O(\lambda)\\
&=\left| \sum_{i=0}^K c_i (\freq_i(P)-\freq_i(Q)) \right|+O(\lambda)
\leq C \cdot \| \freq(P)-\freq(Q) \|_1+ O(\lambda).
\end{align*}
Taking supreme over $f$ on both sides of the above inequality yields the lemma.
\end{proof}


\begin{lemma}
\label{lm:approxpoly1}
For any function $f\in \lip[0,1]$, there exists
$K+1$ real numbers
$c_0,\ldots,c_K \in [-C,C]$ with $C=O(1)$ such that
$f(x)=\sum_{i=0}^K c_i B_{i,K}(x) \pm O(1/\sqrt{K})$ for all $x\in [0,1]$.
\end{lemma}

\begin{proof}
Let $B_K f(x) = \sum_{i=0}^K f(i/K) B_{i,K}(x)$
be the Bernstein polynomial approximation of $f$.
It is known that $B_K f$ converges to $f$ uniformly with the following rate
for $f\in \lip[0,1]$:
$\|B_K f - f\|_\infty \leq O(1/\sqrt{K})$ (see e.g., \cite{rivlin2003introduction}).
\end{proof}

\begin{lemma}
\label{lm:approxpoly}
For any function $f\in \lip[0,1]$, there exists
$K+1$ real numbers
$c_0,\ldots,c_K \in [-C,C]$ with $C=\poly(K)\cdot 2^K$ such that
$f(x)=\sum_{i=0}^K c_i B_{i,K}(x) \pm O(1/K)$ for all $x\in [0,1]$.
\end{lemma}
\begin{proof}
By Jackson's theorem (see Theorem~\ref{thm:jackson}) in approximation theory,
for any function $f\in \lip[0,1]$,
there exist $\{t_i\}_{i\in [K]}$ (with $|t_i|\leq \poly(K)$) such that
$
f(x)=\sum_{i=0}^K t_i T_i(x) \pm O(1/K) \, \forall x\in [0,1],
$
where $T_i$s are Chebyshev polynomials of degrees at most $K$.
Since $\{T_i\}_{i\in [K]}$ and $\{B_{i,K}\}_{i\in [K]}$ are two different bases of the linear space of all polynomials of degree at most $K$,
there is a linear transformation $M$
that can change from one basis to another basis:
For an arbitrary polynomial $P(x)$ of degree at most $K$, we can write
$
P(x)=\sum_{i=0}^K c_i B_{i,K}(x) = \sum_{i=0}^K t_i T_i(x),
$
where $c_i=\sum_{k=0}^K M_{ik} t_k$. Using $t=(t_0,\ldots, t_K)^T$
and $c=(c_0,\ldots, c_K)^T$, we have that
$
c=Mt.
$
It is known that for all $i,j$,
$
|M_{ij}| = (2K-1)!! / (2i-1)!! (2K-2i-1)!!
$
where
$n!!=n(n-2)(n-4)\ldots(4)(2)$ for even $n$
and $n!!=n(n-2)(n-4)\ldots(3)(1)$ for odd $n$~\cite{rababah2003transformation}.
Hence, we have that
$$
\|c\|_\infty \leq \|M\|_{\infty\rightarrow\infty} \|t\|_\infty =\left(\max_{0\leq j\leq K} \sum_{i=0}^K |M_{ij}| \right) \|t\|_\infty
\leq \poly(K)\cdot 2^K.
$$
This implies that for any $f\in \lip$, we can also get $c_i$s with $|c_i|\leq \poly(K)2^K$ such that
$
f(x)=\sum_{i=0}^K t_i T_i(x) \pm O(1/K) =\sum_{i=0}^K c_i B_{i,K}(x) \pm O(1/K)
$
for all $x\in [0,1]$.
\end{proof}


\vspace{-1ex}
\paragraph{\boldmath Reconstructing $\mix$.}
Suppose we have a good empirical frequency vector $\tfreq$ which satisfies
$
\| \tfreq-\freq(\mix) \|_1\leq \lambda/C,
$
where $\lambda$ and $C$ are as in Lemma~\ref{lm:freqtotran1}
Now, we show how to reconstruct the mixture $\mix$ approximately.
We propose a simple LP-based algorithm as follows.

We approximate each $B_{i,K}$ by a piecewise constant function $\barB_{i,K}$ in $[0,1]$
such that $\| B_{i,K} -\barB_{i,K} \|_\infty \leq \epsilon'$ for $\epsilon'=O(\kappa)$ ($\kappa$ in Lemma~\ref{lm:closefreq}).
It is easy to see that $O(1/\epsilon')$ pieces suffice (since $B_{i,K}$ is either monotone or unimodal).
We can divide $[0,1]$ into $h=O(K/\epsilon')$ small intervals $[a_0=0,a_1), [a_1,a_2), \ldots, [a_{h-1},a_h=1]$
such that in each small interval $\barB_{i,K}$ is a constant for all $0\leq i\leq K$.
We use $b_{i,j}$ to denote the value of $\barB_{i,K}$ in interval $[a_j, a_{j+1})$.
For each small interval $[a_j, a_{j+1})$, define an variable $z_j$
(think of $z_j$ as the approximation of $\mix([a_j, a_{j+1}))$).
Consider the following linear program $\LP$:
\begin{align}
z\geq 0
\quad\text{and }\quad
\sum_{j=0}^{h-1} z_j =1
\quad\text{and }\quad
\sum_{j=0}^{h-1} b_{i,j} z_j = \tfreq_i \pm \epsilon', \quad \text{ for } i=0, \ldots, K.
\end{align}
It is easy to see that, by Lemma~\ref{lm:closefreq}, $z_j=\mix([a_j, a_{j+1}))$ defined by the original mixture measure $\mix$
is a feasible solution for $\LP$.

On the other hand, any feasible solution of $\LP$ produces a frequency vector
that is close to $\tfreq$:
Suppose $z^\star$ is an arbitrary feasible solution of $\LP$ and
$\hmix$ is any distribution supported on $[0,1]$ that is consistent with $z^\star$
(i.e., $\hmix([a_j,a_{j+1}))=z^\star_j$),
we have that
$$
\freq_i(\hmix)=\int B_{i,K} \d \hmix
=\pm \epsilon' + \int \barB_{i,K} \d \hmix
= \pm \epsilon' + \sum_{j} b_{i,j} \int_{[a_j,a_{j+1})} \d \hmix
= \pm \epsilon' + \sum_{j} b_{i,j} z^\star_i =\tfreq_i \pm 2\epsilon'.
$$

\begin{proofof}{Theorem~\ref{thm:main}}
Combining with Lemma~\ref{lm:closefreq},
we have that
$$
\|\freq(\hmix) -\freq(\mix)\|_1\leq K\|\freq(\hmix) -\freq(\mix)\|_\infty \leq
K(\|\freq(\hmix) -\tfreq\|_\infty+ \|\tfreq-\freq(\mix)\|_\infty )
\leq O(K \kappa).
$$
Then, using Lemma~\ref{lm:closefreq} with $2^{O(K)}$ samples, we can make
$\|\, \freq(P)-\freq(Q) \,\|_1\leq 1/CK$ (recall that $C=\poly(K)2^K$). So, we finally have that
$$
\tran(\hmix, \mix)\leq C\|\freq_i(\hmix) -\freq_i(\mix)\|_1+O(1/K) \leq O(1/K)
\quad \text{ for } \kappa \in O(1/CK^2). \hfill \qedhere
$$
\end{proofof}

\begin{proofof}{Theorem~\ref{thm:naive}}
The proof is the same as that of Theorem~\ref{thm:main}, except that we use
Lemma~\ref{lm:approxpoly1} instead.
In this case, it suffices to use only $\poly(K)$ samples to ensure that $\|\, \freq(P)-\freq(Q) \,\|_1\leq O(1/K)$.
\end{proofof}

\subsection{Learning $k$-spike mixtures}
\label{subsec:kspikeoned}
We now consider the case where $\mix$ is a $k$-spike mixture supported in $[0,1]$,
i.e., is supported on $k$ points in $[0,1]$.
This result will be useful later when we consider mixtures in higher dimensions.
We now use $K$-snapshots only for $K=2k-1$.
Let the $i$-th moment of $\mix$
be
$g_i(\mix)=\int  x^i \mix(\d x) =\sum_{j=1}^k p_j  \alpha^i_j$.
The algorithm is based on an identifiability lemma proved in~\cite{rabani2014learning}
(Lemma~\ref{lm:moment}) and its converse (Lemma~\ref{lm:momentconverse}).

\begin{lemma}[\cite{rabani2014learning}]
\label{lm:moment}
For any two $k$-spike distributions $\mix_1,\mix_2$ supported on $[0,1]$,
$\|g(\mix_1)-g(\mix_2) \|_2\geq \left(\frac{\tran(\mix_1,\mix_2)}{k} \right)^{O(k)}$.
\end{lemma}

\begin{lemma}
\label{lm:momentconverse}
For any two distributions $\mix_1,\mix_2$ supported on $[0,1]$, and $i\in [K]$,
$|g_i(\mix_1)-g_i(\mix_2) | \leq i\cdot  \tran(\mix_1,\mix_2))$.
\end{lemma}

\begin{proof}
For any $i\in [K]$, it is easy to see that $x^i$ is $i$-Lipschitz in $[0,1]$.
Hence, we have
$$
\hspace{3.5cm}
|g_i(\mix_1)-g_i(\mix_2)| = \left|\int x^i \d(\mix_1-\mix_2)\right| \leq i \cdot \tran(\mix_1,\mix_2).
$$
The last inequality is due to Kantorovich-Rubinstein theorem.
\end{proof}

Recall the frequency vector
$\freq_i(\mix)=\int {K\choose i} x^i (1-x)^{K-i} \mix(\d x)
=\sum_{j=1}^k p_j {K\choose i} x^i (1-x)^{K-i}.
$
Define the {\em normalized frequency vector} to be
$
\nfreq_i(\mix)=\int  x^i (1-x)^{K-i} \mix(\d x)
=\sum_{j=1}^k p_j  x^i (1-x)^{K-i}.
$
Let $\pas$ be the $2k\times 2k$ lower triangular Pascal triangle matrix
with non-zero entries $\pas_{ij}={K-i\choose j-1}$ for $0\leq i\leq K$ and $i\leq j\leq K$.
It is not difficult to verify that $g(\mix)=\pas\,\nfreq(\mix)\ $.
It is known that $\|\pas\|\leq 4^k/\sqrt{3}$.
By Lemma~\ref{lm:closefreq}, using $O((k/\epsilon)^{O(k)})$ samples, the empirical frequency vector $\tfreq$
satisfies that
$
\| \tfreq - \freq(\mix)\|_2 \leq (\epsilon/k)^{O(k)}
$
with probability $0.99$.
Let $\tnfreq_i=\tfreq/{K\choose i}$.
Let $\tg= \pas\,\tnfreq$ be the empirical moment vector.

If we can find a distribution $\tmix$ such that
$\|g(\tmix)-g(\mix)\|_2 \leq \left(\epsilon/k\right)^{\Omega(k)}$,
we know, by Lemma~\ref{lm:moment}, that
$
\tran(\tmix, \mix)\leq \epsilon.
$
In order to find such a $\tmix$, we do the following.
$\tmix$ is a $k$-spike distribution
supported on the set of discrete points
$\{0, \tau, 2\tau, \ldots, 1 \}$ where $\tau=\left(\epsilon/k\right)^{\Omega(k)}$.
First, we guess the support of $\tmix$ (there are ${1/\tau\choose k}$ choices).
Then, we solve the following linear program $\LP_1$, where
$x_j$ represents the probability mass placed at point $j\tau\in \support(\tmix)$:
$$
\LP_1:\quad \Bigl| \sum_j  x_j (j \tau)^i - \tg_i\Bigr| \leq O(K\tau), \text{ for all }i\in [K], \quad
\sum_j x_j =1,
\quad x_j \in [0,1], \text{ for all }j
$$

\begin{theorem}
\label{thm:1dkspike}
Using $(k/\epsilon)^{O(k)}\log (1/\delta)$ many $(2k-1)$-snapshot samples, the above algorithm can produce an estimation $\tmix$, which satisfies that
$\tran(\tmix,\mix)\leq \epsilon$ with probability $1-\delta$.
\end{theorem}

\begin{proof}
We know there is a $k$-spike measure $\mix'$ supported on
$\{0, \tau, 2\tau, \ldots, 1 \}$ such that
$\tran(\mix,\mix')\leq \tau$.
Hence, $|g_i(\mix')-g_i(\mix)|< i\tau$
for all $i$, by Lemma~\ref{lm:momentconverse}.
Also,
\begin{align}
\label{eq:closemoment}
\|\tg-g(\mix)\|_2 \leq \| \pas\| \| \tnfreq - \nfreq(\mix)\|_2
\leq \| \pas\| \| \tfreq - \freq(\mix)\|_2
\leq \left(\frac{\epsilon}{k}\right)^{\Omega(k)}. \quad 
\end{align}

Therefore, we have 
$$
|g_i(\mix')-\tg_i| \leq |g_i(\mix')-g_i(\mix)|+|g_i(\mix)-\tg_i|\leq O(i \tau).
$$
This indicates that $\LP_1$ has a feasible solution.
$\tmix$ is a feasible solution of $\LP_1$, hence
$\| g(\tmix)- \tg \|_2\leq O(K^{3/2}\tau)$.
So,
$$
\| g(\tmix)- g(\mix) \|_2\leq \| g(\tmix)- \tg \|_2+\| g(\mix)- \tg\|_2 \leq O(K^{3/2}\tau)\leq \left(\epsilon/k\right)^{\Omega(k)},
$$
which implies the theorem, by Lemma~\ref{lm:moment}.
\end{proof}

\section{\boldmath Learning multidimensional mixtures on $\simplex_n$: a reduction}
\label{sec:reduction}

We now consider the setting where the mixture 
$\mix$ (on $\simplex_n$) is an arbitrary distribution supported in a $k$-dimensional
subspace in $\mathbb{R}^{n}$.
In this section, we use $\tran_1$ and $\tran_2$ to denote
the transportation distances measured in $L_1$ and $L_2$ norm respectively.
For a point $v$ and a set $S$, we use $\Pi_S(v)$ to denote the projection of $v$ to $S$,
i.e., the point in $S$ that is closest to $v$.
We always assume the projection is with respect to $L_2$ distance, unless specified otherwise.
For any arbitrary measure $\mix$ supported on $\R^n$, we use
$\Pi_{S}(\mix)$ to denote the projected measure
defined as
$\Pi_{S}(\mix)(T) = \mix( \Pi_{S}^{-1}(T) )$ for any measurable $T\subseteq S$.

This section provides a reduction from the original learning problem to
to the problem of learning the projected measure in a specific subspace $\sspan(B)$.
Sections~\ref{sec:directlearn} and~\ref{sec:kspike} complement this reduction by devising
algorithms for learning the projected measure $\mix_B:=\Pi_{\sspan(B)}(\mix)$ (for
arbitrary $k$-dimensional $\mix$ and $k$-spike $\mix$ respectively); combining these
algorithms with the reduction of this section yields algorithms for learning $\mix$.
The space $\sspan(B)$ will satisfy several useful properties (Lemma~\ref{lm:vectorinB}).
One particularly useful property is that any unit vector $v\in\sspan(B)$ has
$\|v\|_\infty\leq O(1/\sqrt{n})$
(ignoring factors depending $\epsilon$ and $k$).
This implies that $L_1$ norm and $L_2$ norm in $\sspan(B)$ are quite close up to scaling,
hence allow us to convert bounds between $L_1$ and $L_2$ distances without losing a factor depending on $n$
(otherwise, we typically lose a factor of $\sqrt{n}$).
Furthermore, we can show we do not lose too much by working in $\sspan(B)$
as most of the mass of $\mix$ is very close to $\sspan(B)$.
Suppose we can learn the projected measure $\mix_B$ well.
If we can show $\mix_B$ is close to the original mixture $\mix$ in $\tran_1$ distance, then
$\tmix_B$, a good estimation of $\mix_B$, would be a good estimation of $\mix$ as well.
However, we are not able to show $\mix_B$ and $\mix$ are close enough in general.
Nevertheless, we can prove that a projection of $\mix_B$ to a smaller polytope is close to $\mix$.
Finally, we need to make some small adjustments 
in order to ensure that our estimation $\tmix$ is a valid mixture, as well as a good
approximation of $\mix$ (see Reduction~\ref{spaceredn}).

Before we delve into the details of our reduction, 
we provide some intuition for why we require the subspace $\Span(B)$ to satisfy the
above-mentioned properties and why the standard SVD method does not suffice.
For ease of discussion, we treat $\epsilon$ and $k$ as constants, but $n$ as a parameter
that can be very large. 
Our goal is to obtain $\Span(B)$ of dimension at most $k$ so that if we can learn the
projected mixture $\mix_B$ within 
$\tran_1$-error at most $\epsilon_1$, then
we can learn $\mix$ within 
$\tran_1$-distance at most $\epsilon$. 
We would like $\epsilon_1$ to be independent of $n$
so that the number of $K$-snapshot samples required to estimate $\mix_B$
within $\tran_1$-distance at most $\e_1$ is independent of $n$ (as is the case in
Theorems~\ref{thm:kspace} and~\ref{thm:kspike}).

Suppose first that we know $A$ exactly, and we simply use $\sspan(A)$ as the
subspace. In fact, it is not difficult to learn $\mix=\prod_{A}\mix$
within $L_2$-transportation distance $\epsilon_1$ using a sample size independent of $n$.
This is mainly due to the rotationally-invariant nature of $L_2$, which makes this
equivalent to a learning problem in $\R^k$. 
However, the same is not true for the $L_1$ distance.
Note that we place no assumptions on $A$, so in order to obtain an estimate $\tmix$ with
$\tran_1(\tmix,\mix)\leq \epsilon_1$, we essentially need to 
ensure that $\tran_2(\tmix,\mix)\leq \epsilon_1/\sqrt{n}$;
however, this would require a sample size depending on $n$.
It is precisely to prevent this $\sqrt{n}$-factor loss 
that we require that an $L_2$-ball in our subspace $\sspan(B)$ be close to 
an $L_\infty$-ball (and hence, an $L_1$-ball is ``nearly spherical''). 
This ensures that $\mix_B$ is supported in an $L_2$-ball of radius
$L=O(1/\sqrt{n})$, which makes it possible to learn $\mix_B$ within
$\tran_2$-distance $\e_1/\sqrt{n}$ with sample size independent of $n$, since the desired 
error is $O(L)$. 
The standard SVD method would typically return the subspace spanned by the first few
eigenvectors of $A$; but this suffers from the same problem as when we know $A$ exactly,
since there is no guarantee that an $L_2$-ball in this subspace is close to an
$L_\infty$-ball in this subspace.  

\medskip
We now state the main result of this section. 
We use the following parameters throughout the paper.
The polynomial in the definition of $C$ below depends on the specific problems and
we will instantiate it later.

\begin{align}
\label{eq:parameters1}
C=\poly\left(k, \,\frac{1}{\epsilon}\right), \quad
L=O\left(\sqrt{\frac{k}{n}} \cdot \frac{C}{\epsilon}\right), \quad
\epsilon_1=O\left(\frac{\epsilon^2}{\sqrt{k}C}\right).
\end{align}

\begin{theorem}
\label{thm:reduction}
Suppose $\mix$ is an arbitrary mixture on $\sspan(A)\cap\simplex^n$ where
$\sspan(A)$ is a $k$-dimensional subspace.
We can find a subspace $\sspan(B)$ of dimension $h\,\, (h\leq k)$ in polytime such that:
\begin{list}{{\upshape (\roman{enumi})}}{\usecounter{enumi} \topsep=0ex \itemsep=0ex
    \addtolength{\leftmargin}{-1.5ex} \addtolength{\labelwidth}{\widthof{(ii)}}}
\item 
$\sspan(B)$ satisfies all properties stated in Lemma~\ref{lm:vectorinB} (see below); and
\item 
If 
we can learn an approximation $\tmix_B$ (supported on $\sspan(B)$) for
the projected measure $\mix_B=\Pi_{\sspan(B)}(\mix)$ such that $\tran_1(\mix_B, \tmix_B)\leq \epsilon_1$
using $N_1(n)$, $N_2(n)$ and $N_{K}(n)$ 1-, 2-, and $K$-snapshot samples,
then 
we can learn a mixture $\tmix$ such that $\tran_1(\mix, \tmix)\leq \epsilon$ using
$O(N_1(n/\epsilon)+n\log n/\epsilon^3)$,
$O(N_2(n/\epsilon)+O(k^4 n^{3}\log n/\epsilon^6))$ and $O(N_{K}(n/\epsilon))$  1-, 2-, and
$K$-snapshot samples respectively.
\end{list}
\end{theorem}

\paragraph{The reduction and its analysis.}
Let $r$ be the vector encoding the $1$-snapshot distribution of $\mix$, i.e.,
$r_i= \Pr[\,\text{the 1-snapshot sample is } i\,]= \int x_i \mix(\d x)$.
We say that the mixture $\mix$ is {\em isotropic}, if
$r_i\in [1/2n,2/n]$.
%
Using $O(n\log n)$ 1-snapshot samples,
we can get sufficiently accurate estimates of $r_i$s with high probability.

\begin{lemma}[\cite{rabani2014learning}]
\label{lm:expectation}
For every $\sigma>0$, we can use $O(\frac{1}{\sigma^3}n\log n)$ independent 1-snapshot samples
to get $\tr_i$ such that, with probability at least $1-1/n^2$,  for all $i\in [n]$,
$$
\tr_i\in (1\pm\sigma) r_i  \quad \forall i \text{ with } r_i\geq \sigma/2n,\quad\quad
\tr_i\leq (1+\sigma)\sigma/2n  \quad\forall i \text{ with } r_i< \sigma/2n.
$$
\end{lemma}

Next, we show it is without loss of generality to assume that
the given mixture is isotropic, at the expense of a small additive error.
The argument essentially follows that of \cite{rabani2014learning},
but is simpler.

\begin{lemma}
\label{lm:isotropy}
Suppose we can learn with probability $1-\delta$ an isotropic mixture on $[n]$
within $L_1$ transportation distance $\epsilon$ using
$N_1(n),N_2(n)$ and $N_{K}(n)$ 1-, 2-, and $K$-snapshot samples respectively.
Then we can learn,
with probability $1-O(\delta)$,
an arbitrary mixture
within $L_1$ transportation distance $2\epsilon$
using
$O(\frac{1}{\sigma^3}n\log n+N_1(n/\sigma)),O(N_2(n/\sigma))$ and $O(N_{K}(n/\sigma))$  1-, 2-, and $K$-snapshot samples respectively,
where $\sigma<\epsilon/4$.
\end{lemma}

From now on, we assume that the given mixture $\mix$ is isotropic.
Let $A$ to be the $n\times n$ symmetric matrix encoding
the 2-snapshot distribution of $\mix$; i.e.,
$A_{ij}$ is the probability of obtaining a 2-snapshot $(i,j)$.
It is easy to see that
$A=\int_{\simplex^n} xx^T  \mix(\d x)$.
Note that the support $\support(\mix)$ of the mixture $\mix$
is contained in the subspace, $\sspan(A)$, spanned by the columns of $A$.
For ease of exposition, we first assume that we know $A$ exactly.
This assumption can be dropped via somewhat standard matrix perturbation arguments, which
we sketch at the end of this section.
%
Consider the hypercube $\calH=[-C/n, C/n]^n$ in $\R^n$
($C$ only depends on $k$ and $\epsilon$, and is fixed later).
We now have all the notation to give a detailed description of the reduction. 

\newpage
{\small
\medskip \hrule \vspace{-2pt}
\begin{reduction} \label{spaceredn}
\end{reduction}

\vspace{-4ex}
\paragraph{\small \boldmath Constructing the basis $B$.}
{\em Input:} Matrix $A$. \quad
{\em Output:} A basis $B$ satisfying Lemma~\ref{lm:vectorinB}.

Consider the centrally symmetric polytope $\calP=\calH\cap \sspan(A)$ and the John ellipsoid $\calE$ inscribed in $\calP$.
It is well known that $\calE\subseteq \calP \subseteq \sqrt{k}\calE$.
Suppose the principle axes of $\sqrt{k}\calE$ are $\{e_1,\ldots, e_k\}$, sorted in nondecreasing order of their lengths.
We choose
the orthonormal basis $B$ to be
$B=\left\{b_i=\frac{e_i}{\|e_i\|_2}  : \|e_i\|_2\geq \frac{\epsilon}{\sqrt{n}}\right\}$.
For every $b_i\in B$, it is easy to see that
$
\|b_i\|_\infty=\frac{\|e_i\|_\infty}{\|e_i\|_2}\leq \frac{C\sqrt{k}}{n} \cdot
\frac{\sqrt{n}}{\epsilon} =O\left(\sqrt{\frac{k}{n}} \cdot \frac{C}{\epsilon}\right)$.

\vspace{-1ex}
\paragraph{\small Final adjustment.}
{\em Input:} Matrix $B$, $\tmix_B$ (which is an approximation of $\mix_B$ and supported on $\sspan(B)$).
\\
{\em Output:} The final estimation $\tmix$ of the original mixture $\mix$.

\begin{list}{\arabic{enumi}.}{\usecounter{enumi} \topsep=0ex \itemsep=0ex
    \addtolength{\leftmargin}{-2ex}}
\item Define the polytope
$\calQ= \left(\simplex^n+\B^n_1(\epsilon)\right) \cap \sspan(B)$.
Here $\B^n_1(\eps)$ denotes the $L_1$-ball in $\R^n$ with radius $\eps$,
and the Minkowski sum $A+B$ of sets $A$ and $B$ is the set $\{a+b\mid a\in A, b\in B\}$.
Essentially, $\calQ$ is the set of points in $\sspan(B)$ with $L_1$ norm
within $[1-\epsilon, 1+\epsilon]$.
\item
Let $\tmix_\calQ=\Pi_{\calQ}(\tmix_B)$
be the measure $\tmix_B$ projected to $\calQ$, i.e.,
$\tmix_\calQ(S) = \tmix_B(\proj{\calQ}^{-1}(S))$ for any $S\subseteq \calQ$.
\item
Notice that $\tmix_\calQ$ may not be a valid mixture since some points in $\tmix_\calQ$ may not be in $\simplex^{n}$.
In this final step, we $L_1$-project $\tmix_\calQ$ back into $\simplex_{n}$ and obtain a valid mixture $\tmix$
(i.e., for each point in $\calQ$, we map it to its $L_1$-closest point in $\simplex^n$), which is our final estimation
of $\mix$.
\end{list}
\hrule
}

\bigskip
Lemma~\ref{lm:C} shows that for large enough $C$, $\calH$  contains
$(1-\epsilon)$ unit of mass of $\mix$.
Lemma~\ref{lm:vectorinB} proves various properties about $\sspan(B)$, which we exploit
to prove that the final adjustment procedure returns a good estimate of $\mix$.

\begin{lemma}
\label{lm:C}
For any $\epsilon>0$, the following hold. 
{\upshape (i)} Suppose $\mix$ is a $k$-spike distribution.
For $C\geq 3k/\epsilon$,
$\mix(\calH)\geq 1-\epsilon$.
{\upshape (ii)} Suppose $\mix$ is an arbitrary distribution supported in a $k$-dimensional subspace.
For $C\geq 5k^2/\epsilon$,
$\mix(\calH)\geq 1-\epsilon$.
\end{lemma}

\begin{proof}
We prove the first statement.
Suppose $\mix=\sum_{i=1}^k p_i \delta_{\alpha_i}$ where $\delta_{\alpha_i}$ is the Dirac delta at point $\alpha_i$.
We use $\alpha_{ij}$ to denote the $j$th coordinate of $\alpha_i$.
Since $\mix$ is isotropic, we know that $\sum_{i=1}^k p_i \alpha_{ij} =r_j \in [1/2n,2/n]$.
So, if $\alpha_{ij}>C/n$ for some $j$ (or equivalently $\alpha_i\notin \calH$), we have $p_i\leq 2/C$.
The lemma thus follows since there can be at most $k$ such points.

To show  the second statement,
consider two convex polytopes
$$
\calP_s=\sspan(A) \cap \frac{1}{k}\calH \quad\text{ and }\quad \calP=\sspan(A) \cap \calH,
$$
where $\frac{1}{k}\calH=[-C/kn,C/kn]^n$.
Both $\calP_1$ and $\calP_2$ are symmetric $k$-dimensional bodies.
By classical result from convex geometry
\footnote{
This can be seen either from John's theorem, or the fact that
Banach-Mazur distance between any two norms in $\R^k$ is at most $k$ (see, e.g., \cite{tomczak1989banach}).
},
we can find a linear transformation $\calK$ of the unit hypercube $[-1,+1]^k$, such that
$\calK\subset \sspan(A)$ and
$$
\calP_s\subseteq \calK\subseteq k\calP_s=\calP.
$$
Now, we confine ourselves in $\sspan(A)$.
$\calK$ has $2k$ faces of codimension 1.
For each such face $F$, consider the polyhedron
$$
\calC_F=\{x \mid x = \alpha y, \text{ for some }\alpha\geq 1 \text{ and } y\in F \}.
$$
In other words, $F$ separates the cone generated by $F$ into two parts and $\calC_F$ is the unbounded part.
We claim that
$\mix(\calC_F)\leq 2k/C$ for any face $F$.
Consider the normalized vector
$
r_F=\int_{\calC_F} x \,\mix(\d x) / \mix(\calC_F).
$
Since $r_F$ is a convex combination of vectors in $\calC_F$ and $\calC_F$ is convex,
$r_F$ is in $\calC_F$.
Moreover, it is easy to see $\calP_s\cap \calC_F=\emptyset$.
So there must be a coordinate of $r_F$ whose value is larger than $C/nk$.
Since $r=\int x \mix(\d x) \geq \mix(\calC_F)r_F$, we must have $\mix(\calC_F)\leq 2k/C$.
All such $\calC_F$s together fully cover the region outside $\calP$,
and there are at most $2k$ such $\calC_F$s.
So the total mass outside $\calP$ is at most $4k^2/C$.
\end{proof}


\begin{lemma}
\label{lm:vectorinB}
Let $L=O\left(\sqrt{k/n} \cdot C/\epsilon\right)$. Let $\calP=\sspan(A) \cap \calH$.
Let $v\in\sspan(B)$.
The following hold. \\
\begin{tabular}{@{}l@{\quad \quad}l}
{\upshape (i)} If $\|v\|_2=1$ then $\|v\|_\infty\leq L$. &
{\upshape (ii)} 
If $\|v\|_1=1$ then $\frac{1}{\sqrt{n}}\leq  \|v\|_2\leq L$. \\
{\upshape (iii)} If $x\in\R^n$ with $\|x\|_1=1$, then $\|\Pi_B(x)\|_2\leq L$. &
{\upshape (iv)} For every point $w\in\calP$,
$\|w-\Pi_B(w)\|_2\leq \epsilon/\sqrt{n}$. \\
\end{tabular}
\end{lemma}

\begin{proof}
Suppose $|B|=h$.
Consider the ellipsoid $\calE_B=\sqrt{k}\calE\cap \sspan(B)$.
Clearly, the principle axes of $\calE_B$ are $e_1,\ldots, e_h$.
Suppose $u$ is an arbitrary point in the boundary of $\calE_B$
and $v=u/\|u\|_2$ is a unit vector in $\sspan(B)$.
Obviously, $\|u\|_\infty\leq C\sqrt{k}/n$ (as $u\in \sqrt{k}\calE\subseteq \sqrt{k}\calH$) and $\|u\|_2\geq \epsilon/\sqrt{n}$.
Hence, $\|v\|_\infty = \|u\|_\infty/\|u\|_2\leq L$, which proves part (i). 

Now we show part (ii). 
The first inequality,
$\frac{1}{\sqrt{n}}\leq  \|v\|_2$, is always true.
To see the second inequality, we use the H\"{o}lder inequality:
$$\|v\|_2^2=\innerprod{v}{v}\leq \|v\|_1 \|v\|_\infty = \frac{\|v\|_\infty}{\|v\|_2}\cdot  \|v\|_2 \leq L \|v\|_2.$$

To prove part (iii), 
use the H\"{o}lder inequality again:
$$
\hspace{4.5cm}
\|\Pi_B(x)\|_2 =
\frac{\innerprod{x}{\Pi_B(x)}}
{\| \Pi_B (x) \|_2}
\leq
\frac{\|x\|_1\|\Pi_B(x)\|_\infty}
{\| \Pi_B (x) \|_2}\leq L.
\hspace{4.7cm}
$$
For part (iv), consider an arbitrary point $w\in \calP=\sspan(A) \cap \calH$.
We can see that $w\in \sqrt{k}\calE$. By the construction of $B$, any point in $\sqrt{k}\calE$
has an $L_2$ distance at most $\|e_{h+1}\|_2$ from $\sspan(B)$, so does $w$.
\end{proof}

We now prove part (ii) of Theorem~\ref{thm:reduction}.
Let $\tmix_B$ supported on $\sspan(B)$ be such that
$\tran_1(\mix_B,\tmix_B)\leq \epsilon_1$.
Define $\mix_\calQ=\Pi_{\calQ}(\mix)$
to be the original measure $\mix$ projected to $\calQ$.

\begin{lemma}
\label{lm:mixandmixQ}
We have that
$
\tran_1(\mix_\calQ , \mix) \leq O(\epsilon).
$
\end{lemma}

\begin{proof}
For any measure $\mu$ and subset $S\subset \R^n$, let $\mu |_S$ be the measure $A$
restricted to $S$.
It is easy to see that
$$
\tran_1(\mix , \mix_\calQ) \leq \tran_1(\mix |_\calH, \Pi_\calQ(\mix |_\calH)) + \tran_1(\mix |_{\overline\calH} , \Pi_\calQ(\mix |_{\overline\calH}))
$$
where $\calH=[-C/n, C/n]^n$ (the hypercube used in Lemma~\ref{lm:C}).
\footnote{
Note that even if two measures are not probability measures, their transportation distance is still well defined
as long as both have the same total mass.
}
Note that even though the transportation distance is measured in $L_1$,
the projection is with respect to $L_2$ distance in this lemma.
We first bound the term $\tran_1(\mix |_{\overline\calH} , \Pi_\calQ(\mix |_{\overline\calH}))$
by coupling every point $p\in \simplex^n$ and $\Pi_\calQ(p)$ together.
By Lemma~\ref{lm:vectorinB} (iv), the $L_2$ distance from every point in $\calP=\sspan(A)\cap \simplex^n\cap \calH$
is at most $\epsilon/\sqrt{n}$ from $\sspan(B)$.
Hence, $\|p-\Pi_B(p)\|_1\leq \sqrt{n}\|p-\Pi_B(p)\|_2\leq \epsilon$
and $\|\Pi_B(p)\|_1\leq \|p\|_1+ \|p-\Pi_B(p)\|_1\leq 1+\epsilon$, which implies $\Pi_\calQ(p)=\Pi_B(p)$.
Thus the first term is at most $\epsilon$.

Now, we bound the second term.
For any point $p\in \simplex^n$, it is easy to see the $L_1$ distance from $p$
to $\Pi_\calQ(p)$ is at most $2+\epsilon$.
Since the total mass in $\mix |_{\overline\calH}$ is at most $\epsilon$,
$\tran_1(\mix |_{\overline\calH} , \Pi_\calQ(\mix |_{\overline\calH}))$ is at most $(2+\epsilon)\epsilon<3\epsilon$.
\end{proof}

\begin{lemma}
\label{lm:tmixQ}
Let
$
\epsilon_1=O\left(\frac{\epsilon^2}{\sqrt{k}C}\right)$.
Let $\tmix_\calQ$ be as defined in Reduction~\ref{spaceredn} and suppose $\tmix_B$
is such that
$
\tran_1(\mix_B, \tmix_B) \leq \epsilon_1$.
Then, it holds that
$\tran_1(\mix_\calQ, \tmix_\calQ) \leq O(\epsilon)$.
\end{lemma}

\begin{proof}
First, we notice that
$\mix_\calQ=\Pi_\calQ(\mix) = \Pi_\calQ (\Pi_{\sspan(B)} (\mix))=\Pi_\calQ(\mix_B)$.
So, we have
$$
\tran_2(\mix_\calQ, \tmix_\calQ)
=\tran_2(\Pi_\calQ(\mix_B), \Pi_\calQ(\tmix_B))
\leq\tran_2(\mix_B, \tmix_B),
$$
where the last inequality holds since $L_2$-projection to a convex set is a contraction
\footnote{
This may not be true for $L_1$ projections.
}
and Lemma~\ref{lm:proj3} (i).
By Lemma~\ref{lm:vectorinB} (ii),
$$
\tran_1(\mix_\calQ, \tmix_\calQ)
\leq \sqrt{n}\tran_2(\mix_\calQ, \tmix_\calQ)
\leq \sqrt{n}\tran_2(\mix_B, \tmix_B)
\leq  \sqrt{n}\cdot L\cdot  \tran_1(\mix_B, \tmix_B).
$$
Plugging in the value $L=O(\sqrt{k/n}\cdot C/\epsilon)$, we prove the lemma.
\end{proof}

\begin{proofof}{part (ii) of Theorem~\ref{thm:reduction}}
By Lemmas~\ref{lm:mixandmixQ} and~\ref{lm:tmixQ}, we have
$\tran_1(\mix, \tmix_\calQ)\leq \tran_1(\mix,  \mix_\calQ)+\tran_1(\mix_\calQ, \tmix_\calQ)
\leq O(\epsilon)$.
By considering the coupling between all points in $\calQ$ and
the corresponding points in $\support(\tmix)$, we can see that
$\tmix$ is the probability measure supported in $\simplex^n$ that
has the closest $L_1$-transportation distance to $\tmix_\calQ$.
Hence, $\tran_1(\tmix, \tmix_\calQ)\leq \tran_1(\mix,\tmix_\calQ)\leq O(\epsilon)$.
We conclude the proof by noting that
$\tran_1(\mix, \tmix)\leq \tran_1(\mix, \tmix_\calQ)
+\tran_1(\tmix_\calQ,\tmix)
\leq O(\epsilon)$.
\end{proofof}

\vspace{-2ex}
\paragraph{\boldmath $A$ is unknown.}
We now remove the assumption that $A$ is known.
First, we obtain a close approximation of $A$ using $O(k^4 n^{3}\log n/\epsilon^6)$ 2-snapshot samples
as follows. 
We choose a Poisson random variable $N_2$ with $\Exp[N_2]=O(k^4 n^{3}\log n/\epsilon^6)$,
choose $N_2$ independent 2-snapshots, and construct a symmetric $n\times n$ matrix $\tA$
where
$\tA_{ii}$ is the frequency of the 2-snapshot $(i,i)$, for all $i\in [n]$, and
$\tA_{ij}=\tA_{ji}$ is half of the total frequency of the 2-snapshots $(i,j)$ and $(j,i)$,
for all $i\ne j$.

\begin{lemma}
\label{lm:tA}
The matrix $\tA$ obtained above with $\Exp[N_2]=O\bigl(\frac{k^4 n^{3}\log n}{\epsilon^6}\bigr)$
satisfies $\|A-\tA\|\leq O\left(\frac{\epsilon^3}{k^2n^{3/2}}\right)$.
\end{lemma}

We find the basis $\tB$ as described in Reduction~\ref{spaceredn},
except that we use $\tA$ instead of $A$.
Since $\tB$ satisfies all properties in Lemma~\ref{lm:vectorinB},
the algorithms and analysis in Sections~\ref{sec:directlearn}, \ref{subsec:project1d}
and~\ref{subsec:reconstruct} continue to work.
Suppose that we have an estimate
$\tmix_{\tB}$ of $\mix_{\tB}=\Pi_{\tB}(\mix)$
such that
$\tran_1(\tmix_{\tB}, \mix_{\tB})\leq \epsilon_1$.
We project $\tmix_{\tB}$ to
$\tcalQ= (1+\epsilon)\simplex^n \cap \sspan(\tB)$ to obtain
$\tmix_\tcalQ$.
The same proof as that of Lemma~\ref{lm:tmixQ} shows that
$
\tran_1(\mix_\tcalQ, \tmix_\tcalQ) \leq O(\epsilon).
$
So the only remaining task is to prove an analogue of Lemma~\ref{lm:mixandmixQ} showing that
$\mix_\tcalQ$ is close to the original mixture $\mix$. 

\begin{lemma}
\label{lm:mixandmixQ2}
We have that
$
\tran_1(\mix_\tcalQ , \mix) \leq O(\epsilon).
$
\end{lemma}

\section{Learning arbitrary mixtures in a $k$-dimensional subspace}
\label{sec:directlearn}

Suppose that $\mix$ is an arbitrary distribution supported on a $k$-dimensional subspace
$\sspan(A)$ in $\mathbb{R}^{n}$.
It is known that in order to learn $\mix$ within transportation distance $\epsilon$, it is
necessary to use $K$-snapshot samples with $K=\Omega(1/\epsilon)$
\cite{rabani2014learning}, even in the 1-dimensional case.
In this section, we generalize the result to higher dimensions.
%
By the reduction in Theorem~\ref{thm:reduction}, we only need to specify how to
learn a good approximation $\tmix_B$ of $\mix_B$ such that $\tran_1(\mix_B, \tmix_B)\leq \epsilon_1$.
This can be done as follows.
$B=\{b_1,\ldots, b_h\}$ is an $n\times h$ matrix
(Recall that $B$ is an orthonormal basis for $\sspan(B)$).
Let $b'_1,\ldots, b'_n$ be columns of $B^T$.
We use the following parameters in this section:
$C=O(k^2/\epsilon)$ as suggested in Lemma~\ref{lm:C},
$\epsilon_1$ and $L$ are as in \eqref{eq:parameters1}, and
\begin{align}
\label{eq:parameters3}
\epsilon_2= \frac{\epsilon_1}{L\sqrt{n}}=\left(\frac{\epsilon}{k}\right)^5, \quad
K=O\left(\frac{h}{\epsilon_2^2} \log \frac{h}{\epsilon_2}\right), \quad \text{ and } \quad
N=O\left(\frac{1}{\epsilon_2}\right)^h.
\end{align}

Suppose we take a $K$-snapshot sample $\sample=\{\ell_1,\ldots, \ell_K\}$ from $\mix$,
where $\ell_i \in [n]$ for $i=1,\ldots, K$.
Let
$
\ep(\sample)=\frac{1}{K}\sum_{i=1}^K b'_{\ell_i}
$
(which is an $h$-vector).
Suppose we have $N$ $K$-snapshot samples $\{\sample_1,\ldots, \sample_N\}$.
We define the empirical measure
$
\ep=\frac{1}{N}\sum_{i=1}^N \delta(\ep(\sample_i))$,
where $\delta()$ is the Dirac delta measure.
Our estimation for $\mix_B$ is the image measure
$
\tmix_B = B\ep=\frac{1}{N}\sum_{i=1}^N \delta(B\ep(\sample_i))$.
Note that $\tmix_B$ is indeed a discrete measure supported on $\R^n$ as $B\ep(\sample_i)$ is an $n$-vector.
We can also see that
$\ep=B^T \tmix_B$ since $B^TB=I$.

\vspace{-1ex}
\paragraph{Analysis.}
First, we define $\mu$ to be the measure $\mix_B$, represented in basis $B$.
Hence, $\mu$ is supported over $\R^h$.
Formally,
$
\mu = B^T \mix_B = B^T \Pi_B \mix = B^T BB^T \mix = B^T \mix$.
Now, we  show that $\ep$ is a good estimation of $\mu$.
For this purpose, we introduce an intermediate measure $\mu_N$ defined as follows:
Suppose the $K$-snapshot sample $\sample_i$ is obtained from distribution $s_i\in \sspan(A)\cap \simplex^n$.
Note that $s_i$ is an $n$-vector and let
$\mix_N=\sum^N_{i=1} \delta(s_i)$ and $\mu_N = B^T \mix_N$.
First, we show $\mu_N$ and $\ep$ are close.

\begin{lemma}
\label{lm:bound1}
Let $\mu_N$ and $\ep$ be defined as above and $K=O(\frac{h}{\epsilon_2^2} \log \frac{h}{\epsilon_2})$. Then,
$
\tran_2(\mu_N,\ep)\leq O(\epsilon_2 L).
$
\end{lemma}
\begin{proof}
We simply couple $B^T s_i \in \support(\mu_N)$ and $\ep(\sample_i)\in\support(\ep')$ together.
Conditioning on $s_i$, we can see that
$
\Exp[\ep(\sample_i)] = B^T s_i$.
Recall from Lemma~\ref{lm:vectorinB} that the magnitude of every entry of $B$ is at most $L$.
By a standard application of the
Chernoff-Hoeffding bound and a union bound over $h$ coordinates, we can see that
$
\Pr[\|\ep(\sample_i)-B^Ts_i\|_\infty > \epsilon_2 L/\sqrt{h} ] < h e^{-2 \epsilon_2^2 K/h} \leq \epsilon_2/2$.
Hence, with high probability, for at least
$(1-\epsilon_2)N$ samples $\sample_i$, we have
$\|\ep(\sample_i)-B^Ts_i\|_2 < \epsilon_2 L$.
Moreover, $\|\ep(\sample_i)-B^Ts_i\|_2 \leq O(L\sqrt{h})$ for all $i$.
So,
$
\tran_2(\mu_N,\ep)\leq (1-\epsilon_2) \cdot \epsilon_2 L + \epsilon_2 \cdot O(L\sqrt{h}) \leq O(\epsilon_2 L).
$
\end{proof}

\begin{lemma}
\label{lm:bound2}
Let $\mu$ and $\mu_N$ be defined as above and $N=O(1/\epsilon_2)^h$.
Then, with probability at least $1-\epsilon_2$, it holds that
$
\tran_2(\mu,\mu_N)\leq O(\epsilon_2 L).
$
\end{lemma}
\begin{proof}
$\mu_N$ is the empirical measure of $\mu$.
It is well known that $\mu_N \rightarrow \mu$ almost surely in the topology of weak convergence.
In particular, the rate of convergence, in terms of transportation distance, can be bounded as follows
\cite{alexander1984probability,yukich1989optimal}:
for any $\epsilon_2$, for $N>C$ for some large constant $C$ depending only on $\epsilon_2$,
with probability at least $1-\epsilon_2$, we have
$\tran_2(\mu_N, \mu)\leq O\left(L/N^{1/h}\right)$.
Plugging $N=O(1/\epsilon_2)^h$ yields the result.
\end{proof}

Combining Lemmas~\ref{lm:bound1} and~\ref{lm:bound2},
we obtain
$
\tran_2(\mu, \ep) = \tran_2(B^T\mix_B, B^T\tmix_B)\leq O(\epsilon_2 L)$.
Viewing $B$ as an operator from $L_2(\R^h)$ to $L_1(\R^n)$,
its operator norm is
$$
\|B\|_{2\rightarrow 1} = \sup_{x\in \R^h} \frac{\|Bx\|_1}{\|x\|_2}
=\sup_{x\in \R^h} \frac{\|Bx\|_1}{\|Bx\|_2} \leq \sqrt{n}.
$$
So by Lemma~\ref{lm:proj3}, 
$\tran_1(\mix_B, \tmix_B)
=\tran_1(B \mu, B\ep)\leq
\|B\|_{2\rightarrow 1} \tran_2(\mu, \ep)
\leq O(\epsilon_2 L \sqrt{n}) \leq \epsilon_1$.

Combining with Theorem~\ref{thm:reduction}, we obtain the following theorem
for learning an arbitrary (even continuous) $k$-dimensional mixture. The
sample size bounds for 1- and 2-snapshots below follow from Lemma~\ref{lm:expectation}
(taking $\sigma=O(\eps)$) and Lemma~\ref{lm:tA}.

\begin{theorem}
\label{thm:kspace}
Let $\mix$ be a mixture supported on $\sspan(A)\cap\simplex_n$, where $\sspan(A)$ is a
$k$-dimensional subspace.
Using $O(n\log n/\eps^3)$, $O(k^4n^3\log n/\eps^6)$, and
$\left(\frac{k}{\epsilon}\right)^{O(k)}$ 1-, 2-, and $K$-snapshot samples respectively, where
$K=\widetilde{O}(k^{11}/\epsilon^{10})$, we can obtain, with probability 0.99, a mixture
$\hmix$ such that $\tran_1(\tmix, \mix)\leq O(\epsilon)$
\end{theorem}

\section{\boldmath Learning $k$-spike mixtures on $\simplex_n$}
\label{sec:kspike}

In this section, we consider the 
setting where $\mix$ is a $k$-spike distribution on $\simplex_{n}$, that is, $\mix$ is
supported on $k$ points in $\simplex_n$. This setting was
also considered in~\cite{rabani2014learning} 
but unlike the results therein, our sample size bounds {\em only depend on $n$ and $k$}
and {\em not} on any ``width'' parameters of $\mix$ (e.g., the least weight
of a mixture constituent, or the distance between two spikes).
We use $K$-snapshot samples only for $K=2k-1$ in this section, which is known to be
necessary~\cite{rabani2014learning}.

The high level idea of our algorithm is as follows.
Again, given the reduction of Section~\ref{sec:reduction}, we only need to provide an algorithm
for learning a good approximation $\tmix_B$ for the projected measure
$\mix_B:=\Pi_{\sspan(B)}(\mix)$.
More specifically, we need
$\tran_1(\tmix_B,\mix_B)\leq \epsilon_1$.
For this purpose, we pick a fine net of directions in $\sspan(B)$
and learn the 1-dimensional projected measures on these directions.
Then we use the 1-dimensional projected measures to reconstruct
$\Pi_{\sspan(B)} \mix$.
The reconstruction can be done by a linear program that is similar to $\LP_1$ in Section~\ref{subsec:kspikeoned}.
The most crucial and technically challenging part is to show that if the 1D-projections of
two measures are close (in $\tran$),
then the two measures must be close as well (Lemma~\ref{lm:projection}). To do this, we
leverage Yudin's theorem (Theorem~\ref{thm:yudin}), which shows that any $\lip$-function $f$
in $\B_2^h(1)$ admits a good approximation in terms of certain 1D-functions with
bounded Lipschitz constant. Since the 1D-projections of the two measures are close, the
Kantorovich-Rubinstein theorem implies that the RHS of \eqref{eq:trans} is
small for these 1D functions, and hence that the RHS of \eqref{eq:trans} is small for
$f$. This implies (again by \eqref{eq:trans}) that the two measures are close in $\tran$.

\begin{theorem}
\label{thm:kspike}
Let $\mix$ be an arbitrary $k$-spike mixture in $\simplex_n$.
Using $O(n\log n/\eps^3)$, $O(k^4 n^{3}\log n/\epsilon^6)$, and $(k/\epsilon)^{O(k^2)}$ 1-
and 2- and $(2k-1)$-snapshot samples respectively, we can obtain, with probability 0.99, a
mixture $\hmix$ such that
$\tran_1(\tmix, \mix)\leq O(\epsilon)$.
\end{theorem}


\subsection{Projecting to one dimension}
\label{subsec:project1d}
Assume $B=\{b_1,\ldots, b_h\}$, where $h=\dim(\sspan(B))\leq k$.
We use the following parameters:
$C=O(k/\epsilon)$ as suggested in Lemma~\ref{lm:C},
$\epsilon_1$ and $L$ are defined as in \eqref{eq:parameters1}, and
\begin{align}
\label{eq:parameters}
K=2k+1, \quad
R=O\left(\frac{h}{\epsilon_1}\right), \quad
\epsilon_2= \epsilon_1^{O(h)}L.
\end{align}

Let $T$ be a set of $n$-dimensional vectors (we call them {\em directions}) in $\sspan(B)$, where
each $t\in T$ is given by $t=\sum_{i=1}^h t_i b_i$ with
$t_i \in \frac{1}{hR}\cdot \{-R,\ldots, R\}$.
In other words,
each direction $t=(t_1,\ldots, t_h)\in T$
has the form $t_i \in \frac{1}{hR}\cdot \{-R,\ldots, R\}$ in basis $B$.
It is easy to see for any $t\in T$, $\|t\|_2\leq 1$.
Consider the set of 1-dimensional ``projected" measures
$\{\mix_t\}_{t\in T}$, where
$\mix_t$ is defined as
$$
\mix_t(S) := \mix(\{x\mid \innerprod{t}{x}\in S\}) \quad \text{for any }S \subset \mathbb{R}.
$$

Now, we show how to estimate the projected measure $\mix_t$ for each $t\in T$.
Since $\|x\|_1=1$ for any $x\in \support(\mix)$, we can see $\mix_t$ is supported within $[-\|t\|_\infty,\|t\|_\infty]$.
Let
$
\phi(x) = \frac{x}{2\|t\|_\infty}+\frac{1}{2}
$
which maps $[-\|t\|_\infty,\|t\|_\infty]$ to $[0,1]$.
Suppose we get a $K$-snapshot sample from the original mixture.
We need to describe how to convert this sample to a $K$-snapshot sample
for the 1-dimensional problem for estimating $\mix_t$.

\begin{enumerate}
\item
For each sampled letter in the $K$-snapshot sample, say the letter is $i\in [n]$,
we get a sample $``1"$ for the 1-d problem with probability $\phi(t_i)$ ($t_i$ is the $i$th coordinate of $t$),
and a sample $``0"$ with probability $1-\phi(t_i)$.
\item
We feed those $K$-snapshot samples to the algorithm for the 1-d problem (see Section~\ref{subsec:kspikeoned})
and obtain a measure $\tmix'_t$.
Our estimation for $\mix_t$ is $\tmix_t$ defined as $\tmix_t(S)= \tmix'_t(\phi(S))$ for any $S\subset [-\|t\|_\infty,\|t\|_\infty]$.
\end{enumerate}
We first need a bound on how good our estimation $\tmix_t$ is.
\begin{lemma}
\label{lm:samplebound1d}
 Using $(kL/\epsilon_2)^{O(k)}=(k/\epsilon)^{O(k^2)}$ many $K$-snapshot samples, the above algorithm can produce,
with probability 0.99,
an estimation $\tmix_t$
such that
$
\tran_2(\tmix_t, \mix_t)\leq \epsilon_2
$
for each $t\in T$.
\end{lemma}
\begin{proof}
Let $\mix'_t$ be the 1-dimensional measure supported on $[0,1]$ defined as
$\mix'_t(S)=\mix_t(\phi^{-1}(S))$ for any $S\subseteq [0,1]$.
A moment reflection shows that $\mix'_t$
is exactly the mixture that generates the converted $K$-snapshot samples (i.e., the $0/1$ samples generated in step 1).
Let $\epsilon'=\epsilon_2/L=\epsilon_1^{O(h)}$.
By Theorem~\ref{thm:1dkspike}, using $(k/\epsilon')^{O(k)}$, the algorithm returns  $\tmix'_t$ with
$\tran(\tmix'_t, \mix'_t)\leq \epsilon'$.
The function $\phi$ stretches the length by a factor of $1/2\|t\|_\infty$ (shifting by a constant does not affect transportation distance),
so
$$
\hspace{3.5cm}
\tran_2(\tmix_t, \mix_t)=\tran(\tmix'_t, \mix'_t)\cdot 2\|t\|_\infty
\leq 2\|t\|_\infty \epsilon'  \leq 2L\epsilon' = \epsilon_2. \qquad \qquad \qedhere
$$
\end{proof}

\subsection{\boldmath Reconstructing $\mix_B$ from the 1D-projections}
\label{subsec:reconstruct}
We use $\Pi_B$ as a short for $\Pi_{\sspan(B)}$
and use $\mix_B$ to denote the projection of $\mix$ to $\sspan(B)$, i.e.,
$\mix_B=\Pi_{B}(\mix)$.
We now reconstruct $\mix_B$ from the 1-dimensional projections $\{\tmix_t\}_{t\in T}$.

Now, we show how to obtain a probability measure $\tmix_B$
such that
$\tran(\tmix_B,\mix_B)\leq O(\epsilon)$.
Let $\supp=\sspan(B)\cap \B^n_{2}(L)$
where $\B^n_{2}(L)$ is the $L_2$ ball in $\R^n$ with radius $L$.
By Lemma~\ref{lm:vectorinB} (iii), $\mix_B=\Pi_B(\mix)$ is supported on $\supp$.
It is well known that there is a $\epsilon_2$-net $\calN$ of size $(L/\epsilon_2)^{O(h)}=(k/\epsilon)^{O(h^2)}$ for $\supp$
(see e.g., \cite{dudley1984course,bourgain1989approximation}),
i.e., for any point $p\in \supp$, there is a point $s\in \calN$ such that $||p-s||_2\leq \epsilon_2$.
Therefore, for any probability measure $\mix$ supported over $\supp$,
there is a discrete distribution $Q$ with support $\calN$ such that $\tran_2(\mix,Q)\leq \epsilon_2$.
Now, we try to find a distribution $Q$ such that
$\tran_1(\tmix_t, Q_t)\leq \epsilon_2$
for each $t\in T$, where $\epsilon_2$ is defined in \eqref{eq:parameters}.
Consider the following linear program ($\LP_2$):
For each point $q\in \calN$, we have a variable $y_q$ ($y_q\geq 0$) corresponding to the probability mass at point $q$
a variable $x_{pq}\geq 0$ representing the mass transported from a point $p\in \support(\tmix_t)$ to $q\in \calN$.
Note that $\tmix_t$ is also a discrete distribution, so
the constraint about the transportation distance can be encoded exactly as a linear program:
\begin{align*}
\LP_2:\quad\quad
& \sum_{p} x_{pq} =y_q \text{ for all }q\in \calN; \quad \\
& \sum_{q}x_{pq}=\tmix_t(\{p\}) \text{ for all }p\in \support(\tmix_t); \quad \\
& \sum_{p,q} |p-\innerprod{q}{t}| \,x_{pq}\leq \epsilon_2; \quad\quad
\sum_q y_q=1.
\end{align*}
Suppose $Q$ is a discrete distribution with support $\calN$ such that $\tran_2(Q,\tmix)\leq \epsilon_2$.
From Lemma~\ref{lm:proj3}, we can see that $\tran(Q_t, \tmix_t)\leq \epsilon_2$ for all $t\in T$ as well
($\innerprod{t}{x}$ for $\|t\|_2\leq 1$ is a contraction). Hence, $\LP$ has a feasible solution.
We obtain a feasible solution $Q$ to $\LP$ and let $\tmix_B=Q$ be our estimate of $\mix_B$.

\paragraph{Analysis.}
Any feasible solution $Q$ to $\LP$
satisfies $\tran(Q_t, \mix_t)\leq \epsilon_2$.
The following crucial lemma asserts that if the corresponding 1-dimensional projections of
two measures are close in transportation distance for every direction, the original
measures must be close too. Thus, we obtain that $\tran_1(\tmix_B,\mix_B)\leq O(\eps_1)$;
combining this with Theorem~\ref{thm:reduction} yields Theorem~\ref{thm:kspike}.

\begin{lemma}
\label{lm:projection}
For any probability measure $P\in \supp $, we use $P_t$ to denote the 1-dimensional
measure $P_t(S):=P(\{x\mid \innerprod{t}{x}\in S\})$ for any $S\subset \mathbb{R}$.
Consider two probability measures $P$ and $Q$ over $\supp$.
If $\,\tran(P_t, Q_t)\leq O(\epsilon_2 )$ for all $t\in T$, then
$
\tran_1(P, Q)\leq O(\epsilon_1).
$
\end{lemma}
\begin{proof}
\eat{
We need the following convergence result about
multidimensional Fourier series expansion:
For a function
$f: [-\pi,\pi]^{h}\rightarrow \mathbb{C}$ (assuming $f$
is $2\pi$-periodic in each axis):
$$
f(x) \sim \sum_{t'\in \mathbb{Z}^{h}} c_{t'}  e^{\i \innerprod{t'}{x}}
$$
where
$c_{t'} = {1 \over (2 \pi)^d} \int_{x\in [-\pi,\pi]^{d}} f(x) e^{-\i \innerprod{t'}{x}}\, \d x $.
The {\em rectangular partial sum} is defined to be
$$
S_{N} f(x) =\sum_{|t'_{1}|\leq N}\ldots\sum_{|t'_{h}|\leq N} c_{t'}  e^{\i \innerprod{t'}{x}}.
$$
We use $|t'|\leq N$ as a shorthand notation for $\sum_{|t'_{1}|\leq N}\ldots\sum_{|t'_{h}|\leq N}$.
It is known that the rectangular partial sum  $S_{N}f(x)$ converges uniformly to $f(x)$ in $[-\pi,\pi]^{h}$
for many function classes as $n$ tends to infinity.
In particular, the following result is known~\cite{alimov92multiple}:
If $f\in \lip$ (in $L_2$ distance), then
\begin{align}
\label{eq:fourier}
|f(x)-(S_Nf)(x)|\le O\Bigl( {\ln^{h} N\over N}\Bigr) \ \ \ \text{for }x\in [-\pi,\pi]^{h}.
\end{align}
}
Consider a function $f$ that is supported on $\supp=\sspan(B)\cap \B^n_{2}(L)$ and $\lip$ in $L_1$ distance
(denoted as $f\in \lip(\supp, L_1)$).
From Lemma~\ref{lm:vectorinB} (ii), we can see $f(x)$ is $\Lip{\frac{1}{L}}$ in $L_2$ distance.
Hence, $f(xL)$, supported on $\supp=\sspan(B)\cap \B^n_{2}(1)$, is $\lip$ in $L_2$ distance.
From now on, let us switch to the representation in basis $B$ for the rest of the proof.
For any $f\in \lip(\supp, L_1)$, using Yudin' Theorem (Theorem~\ref{thm:yudin}) and after scaling, we can see that
there exist $c(t')\in \mathbb{C}$ for $t'\in \mathbb{Z}^{h} \cap \B_2^h(R)$ such that
$|f(x)-(\overline{U}_R f)(x)|\le O(h/ R)$
where
$\overline{U}_{R} f(x) =\sum_{t'\in \mathbb{Z}^{h} \cap \B_2^h(R)} c(t')  e^{\i \innerprod{t'}{x}/L}$

Now, fix some $t\in T$.
In basis $B$, $t'= Rh t$ is an integer vector.
By Kantorovich-Rubinstein theorem, for any $t\in T$,
we have $|\int g \,\d (P_t- Q_t)|\leq \alpha\epsilon_1$
for any $g\in \Lip{\alpha}$ where $\alpha$ is a positive number.
Consider function $e^{\i \alpha x}$ where $\i$ is the imaginary unit.
It is easy to see both its real part and imaginary part are in $\Lip{\alpha}$.
Therefore, we have
$$
\left|\int e^{\i \alpha x} \,\d (P_t- Q_t)\right|\leq   O(\alpha\epsilon_2).
$$
Now, we make a simple but crucial observation that links the projected measure $P_t$ to
the characteristic function of $P$:
$$
\int  e^{\i \innerprod{t'}{x}} \,\d P=\int  e^{\i hR x} \,\d P_t  \quad \text{for any }t\in T \text{ and } t'= Rh t.
$$
In fact, this can be seen from \eqref{eq:change}, by viewing $P_t$ as the image measure of
$P$ under the function
$\innerprod{t}{x}$.

By the Kantorovich-Rubinstein theorem,
$\tran(P,Q) = \sup_{f\in \lip(\supp,L_1)} \left|\int f \d(P-Q) \right|$.
Consider an arbitrary $f\in \lip(\supp,L_1)$. We have that
\begin{align*}
\left|\int f \d(P-Q) \right| & \leq \left|\int \overline{U}_R f \,\d(P-Q)\right| + O\Bigl( {h \over R}\Bigr) \\
& \leq \sum_{t'\in \mathbb{Z}^{h} \cap \B^h(R)}  |c(t')| \cdot \left|\int  e^{\i \innerprod{t'}{x}/L} \,\d(P-Q)\right| + O\Bigl( {h \over R}\Bigr)\\
& = \sum_{t'\in \mathbb{Z}^{h} \cap \B^h(R)}  |c(t)| \cdot \left|\int  e^{\i hRx/L} \,\d(P_t-Q_t)\right| + O\Bigl( {h \over R}\Bigr)\\
&\leq  \frac{hR\epsilon_2}{L}\cdot  \sum_{t'\in \mathbb{Z}^{h} \cap \B^h(R)}  |c(t')|  + O\Bigl( {h \over R}\Bigr)
\end{align*}
\eat{
\begin{align*}
\left|\int f \d(P-Q) \right| & \leq \left|\int \overline{U}_R f \,\d(P-Q)\right| + O\Bigl( {h \over R}\Bigr) \\
& \leq \sum_{t'\in \mathbb{Z}^{h} \cap \B^h(R)}  |\hat{f}(t') \lambda(t'/R)| \cdot \left|\int  e^{\i \innerprod{t'}{x}/L} \,\d(P-Q)\right| + O\Bigl( {h \over R}\Bigr)\\
& = \sum_{t'\in \mathbb{Z}^{h} \cap \B^h(R)}  |\hat{f}(t') \lambda(t'/R)| \cdot \left|\int  e^{\i hRx/L} \,\d(P_t-Q_t)\right| + O\Bigl( {h \over R}\Bigr)\\
&\leq  \frac{hR\epsilon_2}{L}\cdot  \sum_{t'\in \mathbb{Z}^{h} \cap \B^h(R)}  |\hat{f}(t') \lambda(t'/R)|  + O\Bigl( {h \over R}\Bigr)
\end{align*}

We can easily show that $\hat f(t')\leq O((2\pi)^{h} /||t' ||_\infty)$ (using the same proof as \cite[pp.92]{stein2003fourier}).
Moreover, $|\lambda(x)| = |\int_{\B(1)} \phi(y)\phi(x-y)\d y| \leq \|\phi\|^2_2=1$ for any $x\in \B_2^h(1)$.
}
Since $|c(t')| \leq \exp(O(d))$, choosing $R=O(\frac{h}{\epsilon_1})$ and $\epsilon_2=(\epsilon_1/h)^{O(h)}L$, we have that
$
\left|\int f \d(P-Q) \right| \leq O(\epsilon_1).
$
Taking supremum on both sides completes the proof of the lemma.
\eat{
\jiannote{
Can we get rid of the exponential dependency on $h$?
Is there some other specific decomposition other than Fourier decomposition that we can use?
Note that for high dimension Fourier series, we can not hope for a much better bound of the absolute sum of the Fourier coefficients
(the bound has to be exponential on $h$).
Can we use any property about $k$-spike distributions?
Wavelet decomposition??
Or the moment method (Maybe for $k$-spike distributions, the moment method may be the right way to go as only the first 2k-1 moments matters.
The characteristic function contains information about all moments which seems to be wasteful.)?
}
}
\end{proof}

\begin{proofof}{Theorem~\ref{thm:kspike}}
As noted earlier, any feasible solution $Q$ to $\LP$ satisfies
$\tran(Q_t, \mix_t)\leq\epsilon_2$.
By Lemma~\ref{lm:projection} below and noticing that
$$
\mix_t=\Pi_t(\mix)=\Pi_t(\Pi_B \mix)=\Pi_t(\mix_B)=(\mix_B)_t,
$$
we can see that $\tran_1(Q, \mix_B)\leq O(\epsilon_1)$.
Reduction~\ref{spaceredn} and Theorem~\ref{thm:reduction} therefore show that we obtain
$\hmix$ satisfying the stated transportation-distance bound.

The sample size bounds for 1- and 2-snapshots below follow from Lemma~\ref{lm:expectation}
(taking $\sigma=O(\eps)$) and Lemma~\ref{lm:tA} respectively.
Overall, we need to estimate
$R^{O(h)}=(h/\epsilon)^{O(h)}$ many
$\mix_t$s, each requiring
$(k/\epsilon)^{O(k^2)}$ many $(2k-1)$-snapshot samples (by Lemma~\ref{lm:samplebound1d}).
\end{proofof}

\bibliographystyle{plain}
\bibliography{probdb}

\appendix


\section{Proof of Lemma~\ref{lm:proj3}} \label{app:prelim}
Since $\tran(\mu,\nu)\leq \epsilon$,
there exists a coupling $W$ between $\mu$ and $\nu$
such that
$$
\int \|x-y\|_Y \d(W(x,y))\leq \epsilon.
$$
$W$ can also be viewed as a coupling between $T\mu$ and $T\nu$. Therefore,
$$
\hspace{2.0cm}
\tran(T\mu,T\nu) \leq \int \|Tx-Ty\|_Y \d(W(x,y))
\leq \int \|T \|_{X\rightarrow Y}\cdot \|x-y\|_X\, \d(W(x,y))
\leq \|T\|_{X\rightarrow Y} \,\epsilon.
$$
The second statement can be shown in exactly the same way.
The third is as simple.
Suppose $W$ is the optimal coupling between $\mu$ and $\nu$.
Then, we can see that
\begin{align*}
\tran(T\mu, T'\nu)  & \leq \int \| T x - T' x' \|_Y \,W(\d (x,x')) \\
&\leq \int (\| T x - Tx' \|_Y +\| Tx' - T' x' \|_Y) \,W(\d (x,x')) \\
&\leq \|T\|_{X\rightarrow Y} \int \|x-x'\|_X\, W(\d(x,x')) + \| x' \|_X \int \|T-T'\|_{X\rightarrow Y}  \, W(\d(x,x')) \\
&\leq \|T\|_{X\rightarrow Y} \tran(\mu, \nu) + \|x'\|_X\cdot \epsilon =O(\epsilon). \tag*{$\qedsymbol$}
\end{align*}

\section{Proofs from Section~\ref{sec:reduction}} \label{app:reduction}

\begin{proofof}{Lemma~\ref{lm:isotropy}}
With $O(\frac{1}{\sigma^3}n\log n)$ independent 1-snapshot samples,
we can assume that $\tr_i$s satisfy the statement of Lemma~\ref{lm:expectation}.
We modify the mixture as follows:
If there is a letter $i\in [n]$ such that $\tr_i\leq 2\sigma/n$, we simply eliminate this letter.
The total probability of eliminated letters is at most $4\sigma$, which incurs
at most an additive $4\sigma\leq \epsilon$ term in transportation distance.
For each of the remaining letter $i\in [n]$,
we ``split" it into $n_i=\lfloor n\tr_i/\sigma\rfloor$ copies, and the probability of $i$ is equally spit
among these copies.
For the eliminated letter $i$, we can think $n_i=0$.
Let $\hmix$ be the modified mixture.

Consider an $m$-snapshot from the original mixture $\mix$.
If the snapshot includes an eliminated letter, we ignore this snapshot.
Otherwise, each letter $i$ in the snapshot is replaced with one of its $n_i$ copies, chosen
uniformly at random. Then, we feed the algorithm for learning $\hmix$ with this snapshot
(we can easily see the snapshot is distributed exactly the same as one generated from $\hmix$).
Suppose $\tmix$ is an estimate of $\hmix$ (returned by the algorithm).
To obtain an estimate of the original mixture, for each constitute of $\tmix$,
we have a constitute in which the probability of letter $i$ is the sum of the probabilities of the $n_i$ copies.

Now, we show $\hmix$ is isotropic.
Let $n'=\sum_i n_i\leq n/\sigma$ be number of new letters in $\hmix$.
We can see $n'\geq \sum_i \frac{2n\tr_i}{3\sigma} \geq \frac{5n}{8\sigma}$.
For each non-eliminated item $i$, we have
$\tr_i/r_i\in [31/32,33/32]$.
Then, we can easily verify that for each new item $i'$,
we have $\hat{r}_{i'}=\frac{ r_i}{\lfloor n\tr_i/\sigma\rfloor} \in [\frac{1}{2n'}, \frac{2}{n'}]$.
Therefore, $\hmix$ is isotropic.
\end{proofof}


\begin{proofof}{Lemma~\ref{lm:tA}}
Let $D=N_2(\tA-A)$.
It is easy to see that  $\Exp[D_{ij}]=0$.
Moreover, since $N_2$ is a Poisson random variable, $D_{ij}$s are independent of each other.
Let $X^{\ell}_{ij}=1$ if the $\ell$-th snapshot is $(i,j)$.
Let $Y^{\ell}_{ij}=X^{\ell}_{ij}-A_{ij}$.
So, $D_{ij}=\sum_{\ell=1}^{N_2} Y^{\ell}_{ij}$.
We can see that
$$
\Var[D_{ij}\mid N_2=n_2] = n_2 \Var[Y^1_{ij}] \leq n_2 A_{ij} \leq n_2/n.
$$
Let $K=O(k^2 n^{3/2}\log n /\epsilon^3)$.
Using Bernstein's inequality (Proposition~\ref{prop:bernstein}), we can see that
for any $n_2\leq 2\Exp[N_2]$,
$$
\Pr[|D_{ij}|\geq K \mid N_2=n_2] \leq 2\max\left\{\exp\left(- \frac{K^2 n}{n_2}\right),\exp(-3K)\right\} \leq 1-\frac{1}{\exp(n)}.
$$
With a union bound and the fact that $\Pr[n_2\geq 2\Exp[N_2]]\leq 1-\exp(-n)$ , we can see that
with probability $1-\exp(-n)$,
$|D_{ij}|\leq K$ for all $i,j$.
Let $\calE$ denote the event $|D_{ij}|\leq K$ for all $i,j$.
Notice that conditioning on $\calE$,
$D_{ij}$s are still independent of each other.
Moreover,
$$
\Var[D_{ij}\mid \calE] =\Var[D_{ij}\mid |D_{ij}|\leq K] \leq \Var[D_{ij}]=\Exp[N_2] \Var[Y^1_{ij}] \leq \Exp[N_2]/n.
$$
Conditioning on $\calE$, we can apply Theorem~\ref{thm:vu2005spectral} and obtain that, with probability
$1-1/\poly(n)$,
$$
\| D \| \leq 2\sqrt{\frac{E[N_2]}{n}} \cdot \sqrt{n} + O(\sqrt{K} (\Exp[N_2])^{1/4}\ln n)
$$
Plugging in the value of $\Exp[N_2]$ and $K$,
we can see that, with high probability $1-1/\poly(n)$,
$$
\hspace{4cm}
\|\tA-A\| \leq \frac{1}{N_2}\|D\| \leq \frac{2}{\Exp[N_2]} \|D\| \leq O\left(\frac{\epsilon^3}{k^2n^{3/2}}\right).
\qquad \qquad \qedhere
$$
\end{proofof}

\begin{proofof}{Lemma~\ref{lm:mixandmixQ2}}
Suppose  the spectral decomposition of $A$ is
$
A=\sum_{i=1}^k \lambda_i v_iv_i^T.
$
where $\lambda_1\geq \lambda_2\geq \ldots \geq \lambda_k\geq 0$ are the eigenvalues.
Let $\gamma=\epsilon^2/kn$.
Suppose $\gamma\leq \lambda_{k'}\leq \lambda_{k'+1}\leq \ldots \leq \lambda_k$.
It is easy to see that there must a value $k'\leq j\leq k$ such that
$\lambda_{j-1}-\lambda_j \geq \gamma/k$.
Define $A'$ to be the truncation
$$
A'=\sum_{i:i < j} \lambda_i v_iv_i^T.
$$
First, we can see from the definition of $A$ that for any $i$,
$$
\lambda_i=\innerprod{v_i}{A v_i} = \int \innerprod{v_i}{x}^2 \mix(\d x)
=\int \innerprod{v_i}{x}^2 \mix(\d x).
$$
Then, we have that
\begin{align*}
\int \|x-\proj{A'}x \|_2^2 \mix(\d x) & \leq
 \int \Bigl\|\sum_{i:i\geq j}  \innerprod{v_i}{x} v_i \Bigr\|_2^2 \mix(\d x) \\
& = \int \sum_{i:i\geq j} \innerprod{v_i}{x}^2 \mix(\d x)
 \leq \int \sum_{i:\lambda_i< \gamma} \innerprod{v_i}{x}^2 \mix(\d x) \\
& = \sum_{i:\lambda_i< \gamma} \lambda_i  \leq \epsilon^2/n.
\end{align*}
Now, we can bound the transportation distance between $\mix$ and $\tmix$, using Cauchy-Schwarz, as follows:
\begin{align*}
\tran_2(\mix, \proj{A'}\mix) & \leq \int \| x - \proj{A'}x \|_2 \mix(\d x) \\
& \leq \left(\int \| x - \proj{A'}x \|_2^2 \mix(\d x) \int 1 \mix(\d x) \right)^{1/2} \\
& \leq  O(\epsilon/\sqrt{n}).
\end{align*}
Suppose $\tA$ has the spectral decomposition
$
\tA=\sum_{i=1}^{k'} \eta_i u_iu_i^T
$
and
$
\tA'=\sum_{i: i<j} \eta_i u_iu_i^T.
$
Note that
$\proj{\tA'}\mix= \proj{\tA'}(\proj{\tA}\mix)$.
Exactly the same proof also shows that
\begin{align*}
\tran_2(\proj{\tA}\mix, \proj{\tA'}\mix)
 \leq  O(\epsilon/\sqrt{n}).
\end{align*}
All nonzero eigenvalues of $A'$ and $\tA'$ are at least $\epsilon^2/(kn)$.
Let $\Phi$ be the matrix of canonical angles between $\sspan(A')$ and $\sspan(\tA')$.
Using Wedin's Theorem (Theorem~\ref{thm:wedin}) and since
$\|A-\tA\|\leq O\left(\frac{\epsilon^3}{k^2n^{3/2}}\right)$, 
we can see that
\begin{align*}
\tran_2(\proj{A'}\mix, \proj{\tA'}\mix)
& \leq \int \| \proj{A'}(x)- \proj{\tA'}(x) \|_2 \mix(\d x)
\leq \int \| \proj{A'}- \proj{\tA'}\| \cdot \|x\|_2 \mix(\d x) \\
& \leq \int \| \sin \Phi \|_2  \mix(\d x)
 \leq  \| \sin \Phi \|_2   \\
&\leq  \frac{\|A-\tA \| }{\lambda/k}
\leq  O(\epsilon/\sqrt{n}).
\end{align*}
Combining the above inequalities,
we can see that
$$
\tran_1(\mix, \mix_{\tA})\leq \sqrt{n} \tran_2(\mix,\mix_{\tA}) \leq O(\epsilon).
$$
To show that
$\tran_1( \mix_{\tA},\mix_{\tcalQ})\leq \epsilon$, we can use exactly the same proof as
that of Lemma~\ref{lm:mixandmixQ}.
\end{proofof}


\end{document}